\documentclass[accepted]{uai2026} 
                      
\usepackage[american]{babel}

\usepackage{hyperref}
\usepackage{url}

\usepackage{amsmath,amssymb,amsthm,mathtools}
\usepackage{thm-restate}
\usepackage{thmtools} 

\usepackage{here}
\usepackage{graphicx}

\usepackage{booktabs}  
\usepackage{algorithm}
\usepackage[noend]{algpseudocode}

\usepackage{float}
\usepackage{wrapfig}

\usepackage{enumitem}

\newcommand \indep{\mathop{\perp\!\!\!\!\perp}}                                 
\newcommand \dep{\mathop{\not\perp\!\!\!\!\perp}}

\DeclareMathOperator{\E}{\mathbb{E}}

\DeclareMathOperator{\pr}{\mathrm{P}}
\DeclareMathOperator{\I}{\mathbf{I}}

\newcommand{\biX}{\textbf{\textit{X}}}
\newcommand{\bix}{\textbf{\textit{x}}}
\newcommand{\biY}{\textbf{\textit{Y}}}
\newcommand{\biy}{\textbf{\textit{y}}}
\newcommand{\biZ}{\textbf{\textit{Z}}}
\newcommand{\biz}{\textbf{\textit{z}}}
\newcommand{\biA}{\textbf{\textit{A}}}
\newcommand{\bia}{\textbf{\textit{a}}}
\newcommand{\biYhat}{\hat{\textbf{\textit{Y}}}}
\newcommand{\biyhat}{\hat{\textbf{\textit{y}}}}
\newcommand{\biE}{\textbf{\textit{E}}}

\newcommand{\biD}{\textbf{\textit{D}}}
\newcommand{\biW}{\textbf{\textit{W}}}

\newcommand{\biT}{\textbf{\textit{T}}}

\newcommand{\vDAG}{G_{\mathrm{v}}}

\newcommand{\vNode}{\mathcal{V}}
\newcommand{\cNode}{\mathcal{C}}

\newcommand{\cDir}{\mathcal{E}}

\newcommand{\cDAG}{G}
\newcommand{\cMEC}{[G]}
\newcommand{\cCPDAG}{\mathcal{G}}

\newcommand{\IA}{\mathcal{A}}
\newcommand{\IAs}{\mathbf{\mathcal{A}}}
\newcommand{\typeMarg}{\mathsf{marg}}
\newcommand{\typeCond}{\mathsf{cond}}
\newcommand{\typeNever}{\mathsf{never}}

\newcommand{\biC}{\textbf{\textit{C}}}

\newcommand{\npp}{M}
\newcommand{\ipp}{m}

\newcommand{\biP}{\textbf{\textit{P}}} 
\newcommand{\biQ}{\textbf{\textit{Q}}} 
\newcommand{\biR}{\textbf{\textit{R}}} 
\newcommand{\biV}{\textbf{\textit{V}}} 
\newcommand{\biU}{\mathbf{U}}

\newcommand{\myra}{i}
\newcommand{\myrb}{i\hspace{-1pt}i}
\newcommand{\myrc}{i\hspace{-1pt}i\hspace{-1pt}i}
\newcommand{\myrd}{i\hspace{-1pt}v}
\newcommand{\myre}{v}
\newcommand{\myrf}{v\hspace{-1pt}i}

\algnewcommand\algorithmicforeach{\textbf{for each}}
\algdef{S}[FOR]{ForEach}[1]{\algorithmicforeach\ #1\ \algorithmicdo}

\theoremstyle{plain}

\declaretheoremstyle[
  headpunct={:},
  headfont=\bfseries,          
  bodyfont=\itshape,           
  notefont=\normalfont,  
]{amsthmstyle}

\declaretheorem[
  style=amsthmstyle,
  name=Theorem,
  numberwithin=section
]{theorem}

\newtheorem{lemma}[theorem]{Lemma}

\theoremstyle{definition}
\newtheorem{definition}[theorem]{Definition}
\newtheorem{assumption}[theorem]{Assumption}
\theoremstyle{remark}
\newtheorem{remark}[theorem]{Remark}
\declaretheorem[
  style=amsthmstyle,
  sibling=theorem,
  name=Assumption,
  numberwithin=section
]{asmp-restate}

\declaretheorem[
  style=amsthmstyle,
  sibling=theorem,
  name=Theorem,
  numberwithin=section
]{thm-restate}

\usepackage{xcolor}
\definecolor{navyblue}{rgb}{0.0, 0.0, 0.5}
\hypersetup{
    colorlinks=true,
    allcolors=navyblue
}
\usepackage[capitalize,noabbrev]{cleveref}
\renewcommand{\eqref}[1]{(\ref{#1})}

\usepackage{natbib} 
\title{Fairness under Graph Uncertainty: 
Achieving Interventional Fairness \\
with Partially Known Causal Graphs over Clusters of Variables}

\author[1]{{Yoichi Chikahara}{}}
\affil[1]{%
    NTT Communication Science Laboratories\\
    Kyoto\\
    Japan
}

\begin{document}
\maketitle

\begin{abstract}
  Algorithmic decisions about individuals require predictions that are not only accurate but also fair with respect to sensitive attributes such as gender and race. Causal notions of fairness align with legal requirements, yet many methods assume access to detailed knowledge of the underlying causal graph, which is a demanding assumption in practice. 
  We propose a learning framework that achieves interventional fairness by leveraging a causal graph over \textit{clusters of variables}, which is substantially easier to estimate than a variable-level graph. 
  With possible \textit{adjustment cluster sets} identified from such a cluster causal graph, 
  our framework trains a prediction model by reducing the worst-case discrepancy between interventional distributions across these sets. 
  To this end, we develop a computationally efficient barycenter kernel maximum mean discrepancy (MMD)
  that scales favorably with the number of sensitive attribute values. 
  Extensive experiments show that our framework strikes a better balance between fairness and accuracy than existing approaches, 
  highlighting its effectiveness under limited causal graph knowledge.
\end{abstract}

\section{Introduction}\label{sec:intro}

Machine learning is increasingly used to automate high-stakes decisions about individuals, such as hiring and lending \citep{houser2019can,fuster2022predictably}. However, the substantial societal impact of such algorithmic decisions raises concerns about \textit{discrimination} with respect to sensitive attributes such as gender and race. Because discrimination is often legally assessed based on the causal relationship between a challenged decision and sensitive attributes, many causality-based fairness notions have been proposed \citep{kusner2017counterfactual,salimi2019interventional,wu2019pc}.

Nevertheless, achieving causality-based fairness remains challenging. Most existing methods assume access to the ground-truth causal graph over an individual’s features \citep{kusner2017counterfactual,chiappa2018path,salimi2019interventional,chikahara2021learning}. In the absence of perfect domain knowledge, this assumption forces us to solve the notoriously difficult problem of learning the entire causal graph structure from observational data—even though our ultimate objective is simply to ensure fairness.
Moreover, learning the full graph relies on restrictive functional assumptions about the data-generating process \citep{glymour2019review}, and their violations can undermine fairness guarantees.

To relax these demanding assumptions, recent methods leverage a \textit{completed partially directed acyclic graph} (CPDAG), which is identifiable from conditional independence relations \citep{li2024local,zuo2024interventional}. 
These methods aim to achieve a causality-based notion called \textit{interventional fairness} \citep{salimi2019interventional} 
by inferring interventional distributions of the predictions under interventions on sensitive features based on a CPDAG. However, estimating a CPDAG remains challenging in high-dimensional settings: it typically requires a large number of conditional independence tests, so estimation errors can accumulate and ultimately undermine the reliability of fairness guarantees.

	\begin{figure*}[t]
		\includegraphics[height=2.85cm]{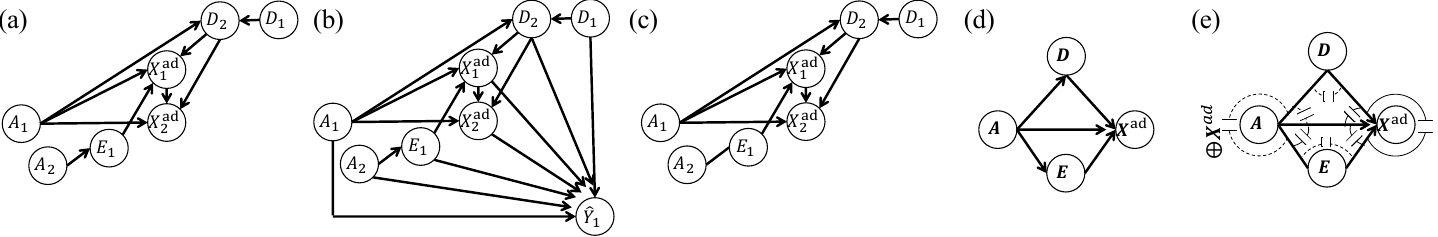}
		\centering 
			\caption{Causal graphs representing a scenario of hiring decisions: (a) DAG over $\biX$, (b) DAG over $\biX$ and $\biYhat$, (c) CPDAG over $\biX$, (d) cluster DAG over $\biX$, and (e) cluster CPDAG over $\biX$ with independence arcs and connection/separation marks.} 
		\label{fig:toy_var_graph}
	\end{figure*}

The goal of this paper is to ensure interventional fairness in more realistic scenarios where we only have access to a \textit{cluster CPDAG} \citep{anand2025causal}, i.e., a CPDAG defined over clusters of variables formed according to a user-specified partition. 
Since such clustering reduces the graph size,
cluster CPDAG inference needs 
substantially fewer conditional independence tests than variable-level CPDAG inference
\citep{tikka2023clustering,anand2025causal}. 
This crucial advantage motivates our research question: \textit{Can we achieve interventional fairness using a cluster CPDAG, despite the causal structure inside each cluster being unknown?}

We answer yes with the following \textbf{key contributions}:
\begin{itemize}
\item We develop a graphical algorithm to enumerate \textit{adjustment cluster sets}, at least one of which enables the identification of interventional distributions
for the true cluster DAG compatible with the cluster CPDAG.
\item We propose a learning framework that achieves interventional fairness by penalizing the \textit{worst-case} unfairness across these adjustment cluster sets. To efficiently measure the overall discrepancy across pairs of interventional distributions, we introduce a computationally efficient barycenter kernel maximum mean discrepancy (MMD) that scales favorably with both the number of sensitive attribute values and the sample size.
\item Through extensive experiments on synthetic and real-world datasets, we demonstrate that the proposed method attains a better trade-off between fairness and predictive accuracy than existing approaches.
\end{itemize}

\section{Preliminaries}\label{sec:preliminaries}

\subsection{Problem Setup} \label{subsec:setup}

We consider (possibly multi-output) supervised learning.
We train a model $h_{\theta}$ with parameters $\theta$ that outputs prediction $\biYhat$ of decision outcome(s) $\biY$ from each individual's features $\biX$, including discrete-valued sensitive features $\biA$.
	
	We seek parameters $\theta$ that achieve a good balance between prediction accuracy and fairness with respect to sensitive features $\biA$. Given $n$ training instances $\{(\textbf{\textit{x}}_i, \biy_i)_{i=1}^n\}$, our learning problem is formulated as 
	\begin{align}
	\underset{\theta}{\text{min}}
	\quad &\frac{1}{n}\sum_{i=1}^n \ell(\biy_i, h_{\theta}(\bix_i)) + \lambda g_{\theta} (\textbf{\textit{x}}_1, \dots, \textbf{\textit{x}}_n),
	\label{eq-obj}
	\end{align}		
	where $\ell$ is a loss function that measures the prediction error, $g_{\theta}$ is a penalty function that quantifies the unfairness of predictive model $h_\theta$, and $\lambda \ge 0$ is a hyperparameter.

We formulate penalty function $g_{\theta}$ based on the notion of interventional fairness \citep{salimi2019interventional}, which is defined using structural causal models (SCMs) \citep{pearl2009causality}.

\subsection{Fairness via Lens of Causality} \label{subsec:causal_notions}

Consider a scenario of making hiring decisions for physically demanding jobs \citep{chikahara2023making}. We predict $\biYhat = [\hat{Y}_1]$ from each applicant’s features $\biX = \{\biA, \biX^{\mathrm{ad}}, \biD, \biE\}$: gender and nationality $\biA = [A_1, A_2]$ (sensitive features), physical test scores $\biX^{\mathrm{ad}} = [X^{\mathrm{ad}}_1, X^{\mathrm{ad}}_2]$, medical test results $\biD = [D_1, D_2]$, and educational background $\biE = [E_1]$. Since the job requires physical strength, we treat the physical test scores $\biX^{\mathrm{ad}}$ as \textit{admissible} features, whose use in decision-making is \textbf{not} regarded as indirect discrimination with respect to the sensitive features $\biA$.

When causal relationships among features $\biX$ are expressed by the \textit{causal graph} in \Cref{fig:toy_var_graph}(a), we can depict the data-generating process of $\biX$ and $\biYhat$ as in \Cref{fig:toy_var_graph}(b) by adding directed edges from each feature in $\biX$ to prediction $\biYhat$.

An SCM formalizes such data-generating process. 
For example, physical test score $X^{\mathrm{ad}}_1$ is determined by the values of its \textit{parents} in the causal graph in \Cref{fig:toy_var_graph}(b):
\begin{align}
X^{\mathrm{ad}}_1 = f_{X^{\mathrm{ad}}_1}(A_1, D_2, E_1, U_{X^{\mathrm{ad}}_1}), \label{eq-scm-xad1}
\end{align}
where $U_{X^{\mathrm{ad}}_1}$ is an exogenous noise, and
$f_{X^{\mathrm{ad}}_1}$ is a deterministic function.
Whereas the true forms of such deterministic functions as $f_{X^{\mathrm{ad}}_1}$  are unknown for observed features $\biX$, the prediction $\biYhat = [\hat{Y}_1]$ is given by the known function, as it is produced by a prediction model 
$h_{\theta}$:
\begin{align}
\biYhat = h_{\theta}(\biX, \biU_{\biYhat}), \label{eq-scm-yhat}
\end{align}
where $\biU_{\biYhat}$ is an exogenous noise term corresponding to the randomness of $h_{\theta}$ when it is a probabilistic model.

Interventional fairness is defined via an operation on an SCM called an \textit{intervention}. For instance, an intervention $do(X^{\mathrm{ad}}_1 = x^{\mathrm{ad}}_1)$ modifies the data-generating process of $X^{\mathrm{ad}}_1$ in \eqref{eq-scm-xad1} by forcing $X^{\mathrm{ad}}_1$ to take the fixed value $x^{\mathrm{ad}}_1$. 
Using such hypothetical modifications to the data-generating process, interventional fairness is defined as the equality of the \textit{interventional distributions} of $\biYhat$ under interventions on the sensitive features $\biA$ and the admissible features $\biX^{\mathrm{ad}}$:
\begin{definition}[\citet{salimi2019interventional}]
Prediction $\biYhat$ is \textit{interventionally fair} with respect to sensitive features $\biA$, if  
\begin{align}
\begin{aligned}  
&\pr(\biYhat = \biyhat \mid do(\biA = \bia), do(\biX^{\mathrm{ad}} = \bix^{\mathrm{ad}}))  \\
= &\pr(\biYhat = \biyhat \mid do(\biA = \bia'), do(\biX^{\mathrm{ad}} = \bix^{\mathrm{ad}})), 
\end{aligned}
\label{eq-int-fair} 
\end{align}
holds for all possible prediction values $\biyhat$, for any discrete values $\bia, \bia'$ of sensitive features $\biA$
and for any discrete values $\bix^{\mathrm{ad}}$ of admissible features $\biX^{\mathrm{ad}}$.
\label{def:int-fair}
\end{definition}
\begin{remark}[\textbf{Connection to other causality-based notions}]
  Unlike \textit{counterfactual fairness} \citep{kusner2017counterfactual,wu2019pc}, interventional fairness is
  (\myra) a \textit{group-level} notion---in contrast to \textit{individual-level} one, which is formalized via conditional distributions conditioned on all features $\biX$ of an individual---and
  (\myrb) defined using interventional distributions,
  not counterfactual distributions \citep{pearl2009causality}.
\end{remark}
An advantage of interventional fairness is that
the interventional distributions in Eq. \eqref{eq-int-fair} can be identified using partial causal graph knowledge, formalized as a CPDAG
\citep{maathuis2009estimating,fang2020ida,guo2022efficient}.

A CPDAG represents a \textit{Markov equivalence class} (MEC), a class of causal DAGs with the same conditional independence relations among variables.
It contains directed edges that are common to all DAGs in the class, and undirected edges whose orientations vary across the DAGs.
For example, the MEC containing the causal DAG in \Cref{fig:toy_var_graph}~(a) is given by the CPDAG in \Cref{fig:toy_var_graph}~(c),
where the directed edges in collider structures (e.g., $A_1 \rightarrow X^{\mathrm{ad}}_1 \leftarrow E_1$ with $A_1$ and $E_1$ being non-adjacent (i.e., \textit{unshielded})) remain directed, since their conditional independence relations differ from those in non-collider structures.

Although CPDAG inference only requires conditional independence tests, it can be challenging in high-dimensional setups due to the large number of tests required. 
This challenge motivates \textit{cluster causal graphs},
as their estimation needs exponentially fewer tests (with respect to the maximum node degree) and can further gain statistical power by leveraging modern multivariate testing \citep{anand2025causal}.

\subsection{Cluster Causal Graphs} \label{subsec:cluster_causal_graphs}

 	\begin{figure}[t]
		\includegraphics[height=5.6cm]{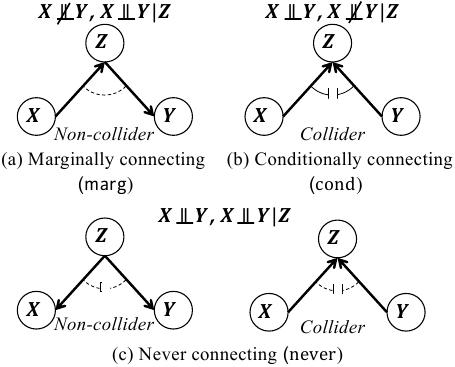}
		\centering 
			\caption{Independence arcs for cluster triplet $\langle \biX, \biZ, \biY \rangle$} 
		\label{fig:ind_arcs}
	\end{figure}

\paragraph{Cluster DAGs} 
A cluster graph $\cDAG$ is constructed 
by grouping nodes $\vNode$ 
in a variable-level DAG $\vDAG$ \citep{anand2023causal}.
Given a partition of $d$ clusters $\cNode = \{\biC_1, \dots, \biC_d\}$ of nodes $\vNode$, 
a cluster graph $\cDAG$ over $\cNode$ has a directed edge $\biC_i \rightarrow \biC_j$ between clusters $\biC_i, \biC_j \in \cNode$, if some $V_i \in \biC_i$ and $V_j \in \biC_j$ have an edge $V_i \rightarrow V_j$ in $\vDAG$.
For instance, the causal graph in \Cref{fig:toy_var_graph}~(a) is converted to the cluster DAG in \Cref{fig:toy_var_graph}~(d).
Following \citet{anand2025causal}, we assume that the resulting cluster graph is a DAG;
this assumption is feasible for many applications 
by leveraging domain knowledge to design an appropriate partition, called an \textit{admissible partition}.

Even under acyclicity,
cluster-level edges alone are \textbf{insufficient} to determine 
conditional independencies among the clusters.
For example, 
as illustrated in \Cref{fig:ind_arcs}~(b) and (c; right),
collider structure $\biX \rightarrow \biZ \leftarrow \biY$ can 
entail different relations
$\biX \dep \biY \mid \biZ$ and $\biX \indep \biY \mid \biZ$, 
depending on whether 
variable-level collider structure $X_i \rightarrow Z_k \leftarrow Y_j$ exists in DAG $\vDAG$
for some variables $X_i \in \biX$, $Y_j \in \biY$, and $Z_k \in \biZ$.

\paragraph{Conditional Independence Representations}
\citet{anand2025causal} represent conditional independencies among clusters using three objects (\Cref{subsec:definition_cluster_causal_graphs}).
An \textit{independence arc} $\IA_{\langle \biX, \biZ, \biY \rangle}$ over an (unshielded) cluster triplet $\langle \biX, \biZ, \biY \rangle$ expresses the conditional independencies as the three possible states: 
\textit{marginally connecting} ($\typeMarg$), \textit{conditionally connecting} ($\typeCond$), and \textit{never connecting} ($\typeNever$).
As in \Cref{fig:ind_arcs},
they have different graphical implications:
While $\typeMarg$ and $\typeCond$ always imply non-collider and collider structures, respectively, $\typeNever$ can arise from either structure.
\textit{Connection} and \textit{separation marks} denote exceptional cases that arise from variable-level paths inside clusters.
A connection mark $\oplus \biC$ and a separation mark $\oslash \biC$ for some cluster set $\biC \subseteq \cNode$
are put on an independence arc $\IA_{\langle \biX, \biZ, \biY \rangle}$, if conditioning on $\biC$ renders $\biX$ and $\biY$ dependent and independent, respectively, which are otherwise not implied by the state of arc $\IA_{\langle \biX, \biZ, \biY \rangle}$.

\paragraph{Cluster CPDAGs} A cluster CPDAG expresses a \textit{cluster MEC}, a class of cluster DAGs with identical conditional independence relations among clusters. 
As in \Cref{fig:toy_var_graph}~(e),
it contains directed edges, independence arcs, and connection/separation marks shared in all cluster DAGs in the class, and undirected edges whose orientations differ in the class.

The \textit{causal learning over clusters} (CLOC) method infers a cluster CPDAG from the data \citep{anand2025causal}.
Compared to other methods for cluster graph inference  
\citep{segal2005learning,tikka2023clustering,wahl2024foundations},
it does not require restrictive assumptions about the underlying graph structure (e.g., the node connectivity within each cluster).

\section{Proposed Method} \label{sec:proposed_method}

We propose a framework for achieving interventional fairness 
using a cluster CPDAG inferred by the CLOC method.

\subsection{Challenges in Cluster CPDAGs} \label{subsec:challenges}

To ensure interventional fairness,
we design penalty function $g_{\theta}$ in \eqref{eq-obj} that reduces the discrepancy between interventional distributions in \eqref{eq-int-fair}.
However, when only a cluster CPDAG is available,
inferring these distributions is challenging.

If we have access to the variable-level causal DAG, 
we can infer interventional distributions 
by identifying adjustment variables $\biZ$ from the graph structure and marginalizing over $\biZ$ and
the remaining variables $\biX^{\mathrm{re}} = \biX \backslash \{\biA, \biX^{\mathrm{ad}}, \biZ\}$:
\begin{align}
    \begin{aligned}
    &\pr(\biYhat = \biyhat \mid do(\biA = \bia), do(\biX^{\mathrm{ad}} = \bix^{\mathrm{ad}}))  \\
    = &\E_{\biZ, \biX^{\mathrm{re}} \mid \bia, \bix^{\mathrm{ad}}}\left[\pr(\biYhat = \biyhat | \biA = \bia, \biX^{\mathrm{ad}} = \bix^{\mathrm{ad}}, \biZ, \biX^{\mathrm{re}}) \right].
    \end{aligned} 
    \label{eq-adjustment} 
\end{align}
This expectation can be estimated from the data distribution.

However,
since a cluster CPDAG represents multiple DAGs (in the cluster MEC), 
we cannot identify such a single adjustment set $\biZ$.
A possible solution is to enumerate adjustment sets $\biZ^{1}, \dots, \biZ^{\npp}$,
such that, for any cluster DAG in the cluster MEC, at least one set enables the adjustment in Eq. \eqref{eq-adjustment}. 

In case of variable-level CPDAGs,
we can efficiently enumerate such adjustment sets \citep{li2024local}, 
thanks to the desirable graphical properties
(e.g., chordality of \textit{chain components}, each connected with undirected edges \citep{guo2022efficient}).
However, cluster CPDAGs do not enjoy these ideal properties, and
their graph structures alone are insufficient to determine conditional independence relations among clusters, as described in \Cref{subsec:cluster_causal_graphs}.

For this reason, we develop a novel adjustment set enumeration algorithm that explicitly accounts for independence arcs and  separation/connection marks in a cluster CPDAG.

\subsection{Adjustment Set Enumeration} \label{subsec:adjustment_sets}

\subsubsection{Assumptions and Enumeration Strategy} \label{subsubsec:overview}

Inferring a cluster CPDAG over features $\biX$ with the CLOC algorithm \citep{anand2025causal} requires three assumptions:
\begin{assumption}
\label{asmp:dag}
The underlying variable-level causal graph over features $\biX$ is a DAG and contains no unobserved confounders between any pair of features in $\biX$.
\end{assumption}
\begin{assumption}
\label{asmp:cpdag}
A partition of clusters $\cNode =\{\biC_1, \dots, \biC_d\}$ is admissible: No directed cycle arises due to the partition. 
\end{assumption}
\begin{assumption}
\label{asmp:faithful}
Distribution $\pr(\cNode)$ is faithful to the underlying cluster DAG over $\cNode$ (See \Cref{subsec:assumption_cluster_causal_graphs} for details).
\end{assumption}
Violating these assumptions implies no guarantee to recover the true cluster CPDAG; however, we empirically observe that the resulting CPDAG is useful for enforcing fairness (\Cref{subsec-add-violate}).
Assuming acyclicity and no unobserved confounders,
which is standard in the literature on interventional fairness \citep{li2024local,zuo2024interventional},
is sufficient to apply the standard \textit{back-door adjustment}  \citep{pearl2009causality}.

For this reason, 
we seek adjustment sets $\biZ^{1}, \dots, \biZ^{\npp}$ such that, for the true cluster DAG, 
at least one of these sets 
\textit{d-separates} all \textit{back-door paths} from sensitive features $\biA$ to prediction $\biYhat$ 
(i.e., the paths with $\biA \leftarrow \dots$ connecting to $\biYhat$).
For this goal, we use the cluster-level d-separation:
\begin{definition}[\citet{anand2025causal}]
\label{def:d-sep}
A path $p = \langle \biC_1, \biC_2$, $\biC_3, \dots, \biC_d \rangle$ in a cluster DAG is d-separated by a set of clusters $\biZ \subset \cNode$ if and only if the sequence of independence arcs along the path, $\IA_{\langle \biC_1, \biC_2, \biC_3 \rangle}, \dots, \IA_{\langle \biC_{d-2}, \biC_{d-1}, \biC_d \rangle}$, contains independence arc $\IA_{\langle \biC_i, \biC_k, \biC_j \rangle}$ for $i, k, j$ that 
\begin{enumerate}
\item is $\typeMarg$ with (a) $\biC_k$ in $\biZ$ or (b) separation mark $\oslash \biC_x$ for some $\biC_x$ on path $p$, where $x \in \{1, \dots, d\}$.
\item is $\typeCond$ with (a) $\biC_k$ and its true descendant $\biC_{\mathrm{td}}$ (with directed path $\biC_k \rightarrow \dots \rightarrow \biC_{\mathrm{td}}$) not in $\biZ$, and (b) has no connection mark $\oplus \biC_x$ such that $\biC_x$ in $\biZ$ or (c) has separation mark $\oslash \biC_x$ for some $\biC_x$ on path $p$. \label{item:cond}
\item is $\typeNever$ and has no connection mark $\oplus \biC_x$ for $\biC_x$ in $\biZ$.
\end{enumerate}
\end{definition}
\Cref{def:d-sep} extends d-separation using independence arcs and marks.
Unlike the original one, 
even when triplet $\langle \biC_i, \biC_k, \biC_j \rangle$ forms a collider structure,
if arc $\IA_{\langle \biC_i, \biC_k, \biC_j \rangle}$ takes $\typeNever$ (\Cref{fig:ind_arcs}(c)),
then the path remains d-separated 
regardless of whether $\biC_k$ (or its descendants) is in $\biZ$.

For a cluster set to d-separate back-door paths from $\biA$ to $\biYhat$,
it suffices to d-separate back-door paths from $\biA$ to each feature in $\biX$ for cluster DAGs over features $\biX$,
as their corresponding cluster CPDAG can be augmented with prediction $\biYhat$, as shown in \Cref{fig:toy_var_graph}(b).
For simplicity, we consider back-door paths $\biA \leftarrow \cdots$; the same procedure applies to $\biX^{\mathrm{ad}} \leftarrow \cdots$. We obtain the adjustment sets in two steps:
\begin{enumerate}
\item \textbf{Parent Enumeration} (\Cref{subsubsec:enum_parents}): 
From possible DAG structures, 
we enumerate the parents of $\biA$, i.e., the clusters that appear first on back-door paths. If all parent sets contain no connection-marked cluster, they achieve adjustment; otherwise, we proceed to Step 2.
\item \textbf{Adjustment Set Completion} (\Cref{subsubsec:additional_clusters}): 
We iteratively augment each adjustment-set candidate with additional clusters needed to d-separate back-door paths.
\end{enumerate}

\subsubsection{Parent Set Enumeration} \label{subsubsec:enum_parents}

Sensitive features $\biA$ can have \textit{definite} and \textit{possible parents}.
A definite parent $\biP \in \mathrm{pa}(\biA)$ is a node with a directed edge $\biA \leftarrow \biP$ 
and is the parent of $\biA$ in all cluster DAGs in the cluster MEC.
By contrast, a possible parent $\biP$ is a node with an undirected edge $\biA - \biP$ that can be oriented as $\biA \leftarrow \biP$ in some but not all cluster DAGs in the cluster MEC.

Letting $\ipp = 1, \dots, \npp$ be the (arbitrary) index for such possible parent sets, 
the parents of $\biA$ are defined as the union of definite and possible parent sets:
\begin{align}
\biZ^{\ipp} \coloneqq \mathrm{pa}(\biA) \cup S^{\ipp}\ \mbox{for}\ \ipp = 1, \dots, \npp, \label{eq-adjustment-cluster-sets}
\end{align}
where $S^{1}, \dots, S^{\npp}$ are possible parent sets, 
with $\npp$ increasing with the number of undirected edges incident to $\biA$.

To enumerate parent sets $\biZ^1, \dots, \biZ^{\npp}$,
we obtain possible parents from the clusters adjacent to $\biA$ via undirected edges 
(i.e., \textit{siblings} $\mathrm{sib}(\biA, \cCPDAG)$).
We give necessary and sufficient conditions 
for cluster set $S$ to be a possible parent set:
\begin{restatable}{theorem}{pptheorem}
\label{thm:possible_parents}
A set of clusters $S \subseteq \mathrm{sib}(\biA, \cCPDAG) \coloneqq \{\biC \in \cNode \colon \biC - \biA \mbox{ in } \cCPDAG\}$ in a cluster CPDAG $\cCPDAG$ is a possible parent set of $\biA$ if and only if it satisfies the following two conditions:
\begin{enumerate}
    \item \textbf{Collider-structure compatibility}: All distinct $\biU,\biV\in S$ satisfy either (\myra) $\biU$ and $\biV$ are adjacent, or (\myrb) $\biU$ and $\biV$ are not adjacent and $\IA_{\langle \biU,\biA,\biV\rangle} = \typeNever$.
    \item \textbf{No back-path condition}: Let $\cCPDAG_{\rightarrow}$ be the directed subgraph of $\cCPDAG$ obtained by removing all undirected edges. For every $\biU \in S$ and every $\biV \in \mathrm{sib}(\biA, \cCPDAG) \backslash S$, there is no directed path from $\biV$ to $\biU$ in $\cCPDAG_{\rightarrow}$.
\end{enumerate}
\end{restatable}  
All proofs are provided in \Cref{proofs}.
Our conditions extend those for variable-level CPDAGs \citep{li2024local}. 
As with the original one, Condition 2 prevents the creation of directed cycles by orienting undirected edges.
A crucial difference lies in Condition 1, which addresses the additional challenge in cluster CPDAGs.
It clarifies which pairs of undirected edges incident to $\biA$ can be oriented as collider structures by taking into account the independence arcs.

If $\biZ^{\ipp}$ has no connection-marked cluster,
it achieves adjustment for some cluster DAGs in the cluster MEC:
\begin{restatable}{theorem}{nocmptheorem}
For any cluster DAG $\cDAG$ in the cluster MEC, let $\biZ$ be the parent set of $\biA$ in $\cDAG$. 
If no cluster in $\biZ$ is annotated by a connection mark in any independence arc for all triples in cluster CPDAG $\cCPDAG$, then $\biZ$ d-separates every back-door paths from $\biA$ in $\cDAG$.
\end{restatable}  
Otherwise, 
we need to complete $\biZ^{\ipp}$ by cluster addition.

\begin{algorithm}[t]
\caption{Adjustment Cluster Completion $(\biZ^{\ipp}, \biA, \cCPDAG)$}
\begin{algorithmic}[1]
\Require $\biZ^{\ipp}$, $\biA$, and cluster CPDAG $\cCPDAG$ over $\biX$
\State $\biZ \gets \biZ^{\ipp}$ in Eq. \eqref{eq-adjustment-cluster-sets}; $\mathrm{Processed}\gets\emptyset$
\State $\biZ_{\oplus}\gets \biZ \cap \{\biC_x\ \mbox{in}\ \oplus\biC_x\ \mbox{in}\ \IA_{\langle\biP, \biQ, \biR \rangle}\ \mbox{for all triples in}\ \cCPDAG\}$
\State $\mathrm{Queue} \gets \{(\biP,\biQ)\colon \biP \in \biZ_{\oplus}; 
\biQ \ \mbox{adjacent to}\ \biP;\biQ \neq \biA\}$
\While{$\mathrm{Queue}\neq\emptyset$}
  \State pop $(\biP,\biQ)$ from $\mathrm{Queue}$
  \If{$(\biP,\biQ)\in \mathrm{Processed}$} \State \textbf{continue} \EndIf
  \State $\mathrm{Processed}\gets \mathrm{Processed}\cup\{(\biP,\biQ)\}$
  \ForEach{$\biR$ adjacent to $\biQ$ in $\cCPDAG$ and $\biR\neq \biP$}
    \If{$\IA_{\langle\biP, \biQ, \biR \rangle}=\typeMarg$}
      \State $\biZ\gets \biZ\cup\{\biQ\}$
    \ElsIf{$\IA_{\langle\biP, \biQ, \biR \rangle}=\typeNever$}
      \State $\biZ\gets \biZ\cup\{\biQ\}$
      \If{Some $\biP_x \in \biZ$ in $\oplus \biC_x$ in $\IA_{\langle\biP, \biQ, \biR \rangle}$}
        \State push $(\biQ,\biR)$ into $\mathrm{Queue}$
      \EndIf
\ElsIf{$\IA_{\langle\biP, \biQ, \biR \rangle}=\typeCond$}
  \State $\biE \gets \{\biQ\}\cup\textsc{PossDesc}(\cCPDAG,\biQ)$; $\biE \gets \biZ\cap \biE$
  \If{$\biE \neq\emptyset$ and $\biZ$ has no cluster in $\oslash \biC_x$ in $\IA_{\langle\biP, \biQ, \biR \rangle}$}
      \State \Return $\mathrm{Unidentifiable}$
  \EndIf
\EndIf
  \EndFor
\EndWhile
\If{\textsc{FinalCert}$(\biZ,\biA,\cCPDAG)$\footnotemark returns $\mathrm{Unidentifiable}$}
  \State \Return $\mathrm{Unidentifiable}$
\EndIf
\State \Return $(\biZ,\mathrm{OK})$
\end{algorithmic}
 \label{alg:add-adj-cluster-set}
\end{algorithm}
\footnotetext{\textsc{FinalCert}$(\biZ,\biA,\cCPDAG)$ reruns lines 2--19 of \Cref{alg:add-adj-cluster-set} with the final $\biZ$ fixed, without updating $\biZ$ in lines 11 and 13. It returns $\mathrm{Unidentifiable}$ if the condition in line 18 is violated.}

\subsubsection{Adjustment Cluster Completion} \label{subsubsec:additional_clusters}

\Cref{alg:add-adj-cluster-set} shows our queue-based propagation procedure for augmenting each set $\biZ$ with the clusters required for d-separation. Using a queue initialized with the parents of $\biA$ that appear in connection marks, it repeatedly inspects local triple annotations and expands $\biZ$ whenever the annotations indicate that additional conditioning clusters are needed to achieve d-separation.
The returned set $\biZ$ is a valid adjustment set for some cluster DAGs in the cluster MEC:
\begin{restatable}{theorem}{cmptheorem}
Fix an enumerated parent set $\biZ^{\ipp}$. If \Cref{alg:add-adj-cluster-set} returns $(\biZ, \mathrm{OK})$ for this input set, then $\biZ$ d-separates every back-door path from $\biA$ for all cluster DAGs in the cluster MEC that realize the parent set $\biZ^{\ipp}$. 
\end{restatable}  

While Condition~\ref{item:cond} in the cluster-level d-separation (\Cref{def:d-sep}) involves \textit{true} descendants, line 17 approximates them as \textit{possible} descendants $\textsc{PossDesc}(\cCPDAG,\biQ)$, which we define as the set of clusters reachable from $\biQ$ by orienting undirected edges away from $\biQ$. This approximation is \textbf{inevitable} because a cluster CPDAG contains undirected edges, making the true descendant set unidentifiable. 
To provide a safeguard against such unidentifiable cases,
\Cref{alg:add-adj-cluster-set} explicitly returns $\mathrm{Unidentifiable}$ (line 19).

\paragraph{Graph Refinement.}
In these $\mathrm{Unidentifiable}$ cases,
we refine the cluster CPDAG 
by splitting clusters that appear in connection marks into singleton nodes, 
and then obtain adjustment sets 
using the refined graph. 
This refinement incurs an additional cost for re-estimating the cluster CPDAG; 
however, it is negligible in practice compared with variable-level CPDAG estimation because the graph size
is typically much smaller 
(see \Cref{subsec-add-graph-inference} for the time complexity).

The number of refinement steps depends on the provided cluster CPDAG (e.g., the node degrees) and is at most the number of clusters $d$,
where we obtain a variable-level CPDAG and hence can always identify valid adjustment sets.
Refinement can also help reduce the number of undirected edges and adjustment sets $\npp$, as  
discussed in \Cref{subsec-add-dense}.

\subsection{Penalizing Worst-Case Unfairness} \label{subsec:penalty_function}

Using the (augmented) adjustment cluster sets $\biZ^{1}, \dots, \biZ^{\npp}$, we now construct the unfairness penalty $g_{\theta}$ in Eq. \eqref{eq-obj}.

Because it is uncertain which $\biZ^{1}, \dots, \biZ^{\npp}$ yields valid adjustment for the true cluster DAG, we consider the \textit{worst-case} unfairness over these sets, 
defined as the maximum of the discrepancy between interventional distributions.
For each $\biZ^{\ipp}$ we aggregate kernel MMDs \citep{gretton2012kernel} between the interventional distribution pairs in \eqref{eq-int-fair}, and then take the maximum across $\ipp = 1, \dots, \npp$:
\begin{align*}
    \max_{\ipp} \sum_{\bix^{\mathrm{ad}}} \sum_{\bia, \bia'} \mathrm{MMD}_{\ipp}(\pr_{\biYhat| do(\bia), do(\bix^{\mathrm{ad}})}, \pr_{\biYhat| do(\bia'), do(\bix^{\mathrm{ad}})}), 
\end{align*}
where  $\mathrm{MMD}_{\ipp}(\cdot,\cdot)$ denotes the kernel MMD between the corresponding interventional distributions identified via $\biZ^{\ipp}$.

Unfortunately, this maximization is computationally prohibitive, requiring $\mathcal{O}(\npp\, N_{\biA}^2\, N_{\biX^{\mathrm{ad}}}\, n^2)$ time, where $N_{\biA}$ and $N_{\biX^{\mathrm{ad}}}$ are the cardinalities of the sensitive and admissible feature values,
and $n$ is the sample size.
To address this issue, 
below we introduce our estimation strategies 
and derive a computationally efficient penalty function $g_{\theta}$.

\subsubsection{Towards Efficient MMD Estimation} \label{subsubsec:efficient_estimation}

Our estimation strategies are twofold. 
First, we reduce the quadratic time complexity $\mathcal{O}(N^2_{\biA})$ to $\mathcal{O}(N_{\biA})$ by 
decomposing the summation over pairwise MMDs for sensitive feature values $\bia$ and $\bia'$ as the sum of \textit{barycenter} MMDs:
\begin{align}
&\sum_{\bia, \bia'} \mathrm{MMD}_{\ipp}(\pr_{\biYhat| do(\bia), do(\bix^{\mathrm{ad}})}, \pr_{\biYhat| do(\bia'), do(\bix^{\mathrm{ad}})}) \label{eq-int-mmd} \\
=&c \sum_{\bia} \mathrm{MMD}_{\ipp}(\pr_{\biYhat| do(\bia), do(\bix^{\mathrm{ad}})}, \frac{1}{N_{\biA}} \sum_{\bia'} \pr_{\biYhat| do(\bia'), do(\bix^{\mathrm{ad}})}), \nonumber
\end{align}
where $c = 2 N_{\biA}$ is a constant,
and $\frac{1}{N_{\biA}} \sum_{\bia'} \pr_{\biYhat| do(\bia'), do(\bix^{\mathrm{ad}})}$ is the mixture distribution corresponding to the barycenter of distributions $\{\pr(\biYhat| do(\bia), do(\bix^{\mathrm{ad}}))\}_{\bia}$.
The equality holds from the fact that the MMD is a distance metric in the inner product space, called the reproducing kernel Hilbert space (RKHS); see \Cref{dev:barycenter_mmd} for the derivation.

Second, we decrease the time for computing each MMD from $\mathcal{O}(n^2)$ to $\mathcal{O}(n d_{\mathrm{RFF}})$ 
by approximating feature mapping $\phi(\cdot)$ of kernel function $k(\biy, \biy') = \langle \phi(\biy), \phi(\biy')\rangle$
as a $d_{\mathrm{RFF}}$-dimensional vector of random Fourier features (RFFs) \citep{rahimi2007random}.
The MMD in \eqref{eq-int-mmd} is defined as
the distance in the RKHS $\mathcal{H}_k$ between \textit{kernel mean embeddings}: 
\begin{align}    
\| \mu_{\biYhat\mid do(\bia), do(\bix^{\mathrm{ad}})} - \mu_{\biYhat\mid do(\bia'), do(\bix^{\mathrm{ad}})} \|_{\mathcal{H}_{k}}, \label{eq-def-mmd}
\end{align} 
where kernel mean embedding $\mu_{\biYhat\mid do(\bia), do(\bix^{\mathrm{ad}})}$ is defined as
\begin{align}
\mu_{\biYhat\mid do(\bia), do(\bix^{\mathrm{ad}})} \coloneqq
\E_{\biYhat\mid do(\biA = \bia), do(\biX^{\mathrm{ad}} = \bix^{\mathrm{ad}})}[\phi(\biYhat)]. \label{eq-def-kme}
\end{align}
RFFs approximate feature mapping $\phi(\cdot)$ in \eqref{eq-def-kme} as a vector:
\begin{align*}
 \phi_{\mathrm{RFF}}(\biy) = 
 \left[\cos(\omega_1^\top \biy + b_1), \dots, \cos(\omega_{d_{\mathrm{RFF}}}^\top \biy + b_{d_{\mathrm{RFF}}})\right]^\top, 
\end{align*}
where $\{\omega_i\}_{i=1}^{d_{\mathrm{RFF}}}$ and $\{b_i\}_{i=1}^{d_{\mathrm{RFF}}}$ are randomly sampled coefficients; with 
the Gaussian kernel $k$ with bandwidth $\gamma > 0$,
they are sampled from the Gaussian distribution $\mathcal{N}(0, \gamma^{-2})$ and the uniform distribution $\mathcal{U}[0,2\pi]$, respectively.

With these two strategies, we can efficiently compute the worst-case MMD with $\mathcal{O}(\npp N_{\biA} N_{\biX^{\mathrm{ad}}} n d_{\mathrm{RFF}})$.

\subsubsection{Penalty Function Formulation} \label{subsubsec:penalty_formulation}

To estimate the barycenter MMD in \eqref{eq-int-mmd},
we develop a weighted estimator of the kernel mean embedding in \eqref{eq-def-kme}, 
by combining \textit{inverse probability weighting} (IPW) and RFFs. 

Recall that interventional distribution $\pr(\biYhat|do(\bia), do(\bix^{\mathrm{ad}}))$ is identified using a valid adjustment set $\biZ$ as the conditional expectation in \eqref{eq-adjustment}.
Using IPW, this can be estimated as
\begin{align}
    \begin{aligned}
    &\E_{\biZ, \biX^{\mathrm{re}} \mid \bia, \bix^{\mathrm{ad}}}\left[\pr(\biYhat = \biyhat | \biA = \bia, \biX^{\mathrm{ad}} = \bix^{\mathrm{ad}}, \biZ, \biX^{\mathrm{re}}) \right], \\
    = &\E\left[\frac{\I(\bia, \bix^{\mathrm{ad}})}{\pi_{\bia, \bix^{\mathrm{ad}}}(\biZ)} \E_{\biX^{\mathrm{re}} \mid \biA, \biX^{\mathrm{ad}}, \biZ} \left[ \pr(\biYhat = \biyhat \mid \biA,  \biX^{\mathrm{ad}}, \biZ, \biX^{\mathrm{re}}) \right] \right],
    \end{aligned} 
    \label{eq-IPW} 
\end{align}
where $\I(\bia, \bix^{\mathrm{ad}})$ is the indicator function that takes $1$ if $\biA = \bia$ and $\biX^{\mathrm{ad}} = \bix^{\mathrm{ad}}$ for discrete $\biA, \biX^{\mathrm{ad}}$; otherwise 0, 
$\pi_{\bia, \bix^{\mathrm{ad}}}(\biZ) \coloneqq \pr(\biA=\bia, \biX^{\mathrm{ad}} = \bix^{\mathrm{ad}} \mid \biZ)$ is the probability known as the \textit{joint propensity score}, 
and $\pr(\biYhat = \biyhat \mid \bia,  \bix^{\mathrm{ad}}, \biZ, \biX^{\mathrm{re}})$ is the distribution of $\biYhat$ given by predictive model $h_{\theta}$. 
We derive Eqs. \eqref{eq-adjustment} and \eqref{eq-IPW} in \Cref{dev:adjustment_formula}.

As with \eqref{eq-IPW}, 
the kernel mean embedding in Eq. \eqref{eq-def-kme}
can be estimated as an empirical weighted average of RFF vectors:
\begin{align}
    \hat{\mu}_{\biYhat\mid do(\bia), do(\bix^{\mathrm{ad}})} = \frac{1}{n} \sum_{i=1}^n w_i \phi_{\mathrm{RFF}}(\biyhat_i), \label{eq-empirical-kme}
\end{align}
where $\biyhat_i = h_{\theta}(\bix_i)$ is prediction, and $w_i$ is an IPW weight:
\begin{align}
w_i = \frac{\I(\biA_i = \bia, \biX^{\mathrm{ad}}_i = \bix^{\mathrm{ad}})}{\pi_{\bia, \bix^{\mathrm{ad}}}(\biz^{\ipp}_i)}, \label{eq-importance-weight}
\end{align}
where $\pi_{\bia, \bix^{\mathrm{ad}}}(\biz^{\ipp}_i) = \pr(\biA=\bia, \biX^{\mathrm{ad}} = \bix^{\mathrm{ad}} \mid \biZ^{\ipp}=\biz^{\ipp}_i)$ is the joint propensity score computed by fitting $\pi_{\bia, \bix^{\mathrm{ad}}}$ to the training data beforehand;
we can apply IPW weight stabilization techniques like clipping/smoothing \citep{chikahara2024differentiable}, 
as empirically tested in \Cref{subsec-add-sensitivity}.

By computing the barycenter kernel MMD in \eqref{eq-int-mmd} with weighted estimators \eqref{eq-empirical-kme},
we obtain our penalty function as
\begin{align}
  \begin{aligned}
    &g_{\theta}(\bix_1, \dots, \bix_n)  \\
    = &\max_{\ipp=1,\dots,\npp} \sum_{\bix^{\mathrm{ad}}} \sum_{\bia} \| \hat{\mu}_{\biYhat\mid do(\bia), do(\bix^{\mathrm{ad}})} - \hat{\bar{\mu}}_{\biYhat\mid do(\bix^{\mathrm{ad}})} \|^2_{2}, 
  \end{aligned}
    \label{eq-unfairness-penalty}
\end{align}
where $\hat{\bar{\mu}}_{\biYhat\mid do(\bix^{\mathrm{ad}})} = \frac{1}{N_{\biA}} \sum_{\bia} \hat{\mu}_{\biYhat\mid do(\bia), do(\bix^{\mathrm{ad}})}$ is the kernel mean embedding corresponding to the empirical barycenter. 
Since $\max_{\ipp=1,\dots,\npp}$ is non-differentiable,
we approximate it as the \textit{Mellowmax} function \citep{asadi2017alternative}, empirically improving performance (\Cref{subsec-add-sensitivity}).

\begin{table*}[t]
  \centering
  \caption{Root mean squared error (RMSE) and unfairness for linear (top) and nonlinear (bottom) datasets on test data}
  \label{table:synth_rmse_unfairness}
  \scshape

  \scalebox{0.93}{%
  \begin{tabular}{lcccccc}
    \toprule
    & \multicolumn{2}{c}{$d$=5 ($d_v$=15)} & \multicolumn{2}{c}{$d$=10 ($d_v$=30)} & \multicolumn{2}{c}{$d$=15 ($d_v$=45)} \\
    \textsc{Linear}
    & \textsc{RMSE}$\downarrow$ & \textsc{Unfairness}$\downarrow$
    & \textsc{RMSE}$\downarrow$ & \textsc{Unfairness}$\downarrow$
    & \textsc{RMSE}$\downarrow$ & \textsc{Unfairness}$\downarrow$ \\
    \midrule
    \textsc{Oracle} 
      & $0.957 \pm 0.037$ & $0.000 \pm 0.000$
      & $0.977 \pm 0.044$ & $0.000 \pm 0.000$
      & $1.069 \pm 0.143$ & $0.000 \pm 0.000$ \\
    \midrule
    \textsc{Full} 
      & $0.523 \pm 0.195$ & $0.259 \pm 0.129$
      & $0.483 \pm 0.195$ & $0.229 \pm 0.136$
      & $0.554 \pm 0.210$ & $0.081 \pm 0.117$ \\
    \midrule
    \textsc{Unaware} 
      & $0.774 \pm 0.154$ & $0.071 \pm 0.059$
      & $0.774 \pm 0.154$ & $0.062 \pm 0.079$
      & $0.793 \pm 0.120$ & $0.046 \pm 0.104$ \\
    \textsc{No-DesCs} 
      & $0.739 \pm 0.140$ & $0.064 \pm 0.089$
      & $0.739 \pm 0.140$ & $0.065 \pm 0.100$
      & $0.699 \pm 0.135$ & $0.033 \pm 0.028$ \\
    $\epsilon$-\textsc{IFair} 
      & $0.647 \pm 0.030$ & $0.069 \pm 0.031$
      & $0.663 \pm 0.062$ & $0.071 \pm 0.011$
      & $0.671 \pm 0.044$ & $0.040 \pm 0.011$ \\
    $\ell$-\textsc{IFair} 
      & $0.875 \pm 0.028$ & $0.069 \pm 0.081$
      & $0.964 \pm 0.049$ & $0.072 \pm 0.052$
      & $0.764 \pm 0.049$ & $0.052 \pm 0.052$ \\
    \textbf{\textsc{C-IFair}} 
      & $0.643 \pm 0.127$ & $0.060 \pm 0.054$
      & $0.660 \pm 0.123$ & $0.056 \pm 0.052$
      & $0.669 \pm 0.171$ & $0.020 \pm 0.036$ \\
    \bottomrule
  \end{tabular}%
  } 
  \scalebox{0.93}{%
  \begin{tabular}{lcccccc}
    \toprule
    & \multicolumn{2}{c}{$d$=5 ($d_v$=15)} & \multicolumn{2}{c}{$d$=10 ($d_v$=30)} & \multicolumn{2}{c}{$d$=15 ($d_v$=45)} \\
    \textsc{Nonlinear}
    & \textsc{RMSE}$\downarrow$ & \textsc{Unfairness}$\downarrow$
    & \textsc{RMSE}$\downarrow$ & \textsc{Unfairness}$\downarrow$
    & \textsc{RMSE}$\downarrow$ & \textsc{Unfairness}$\downarrow$ \\
    \midrule
    \textsc{Oracle} 
      & $1.008 \pm 0.030$ & $0.000 \pm 0.000$
      & $1.055 \pm 0.069$ & $0.000 \pm 0.000$
      & $1.000 \pm 0.073$ & $0.000 \pm 0.000$ \\
    \midrule
    \textsc{Full} 
      & $0.914 \pm 0.066$ & $0.121 \pm 0.073$
      & $0.925 \pm 0.055$ & $0.026 \pm 0.026$
      & $0.928 \pm 0.058$ & $0.021 \pm 0.007$ \\
    \midrule
    \textsc{Unaware} 
      & $0.978 \pm 0.058$ & $0.069 \pm 0.073$
      & $0.957 \pm 0.061$ & $0.027 \pm 0.038$
      & $1.014 \pm 0.055$ & $0.024 \pm 0.017$ \\
    \textsc{No-DesCs} 
      & $0.978 \pm 0.061$ & $0.084 \pm 0.113$
      & $0.950 \pm 0.066$ & $0.022 \pm 0.028$
      & $1.022 \pm 0.069$ & $0.020 \pm 0.008$ \\
    $\epsilon$-\textsc{IFair} 
      & $0.938 \pm 0.006$ & $0.077 \pm 0.009$
      & $0.953 \pm 0.026$ & $0.027 \pm 0.000$
      & $0.962 \pm 0.016$ & $0.023 \pm 0.001$ \\
    $\ell$-\textsc{IFair} 
      & $0.987 \pm 0.065$ & $0.076 \pm 0.042$
      & $0.982 \pm 0.041$ & $0.023 \pm 0.010$
      & $1.042 \pm 0.041$ & $0.017 \pm 0.010$ \\
    \textbf{\textsc{C-IFair}} 
      & $0.929 \pm 0.051$ & $0.066 \pm 0.046$
      & $0.948 \pm 0.051$ & $0.021 \pm 0.035$
      & $0.960 \pm 0.036$ & $0.010 \pm 0.012$ \\
    \bottomrule
  \end{tabular}%
  } 
\end{table*}

 	\begin{figure*}[t]
		\includegraphics[height=3.8cm]{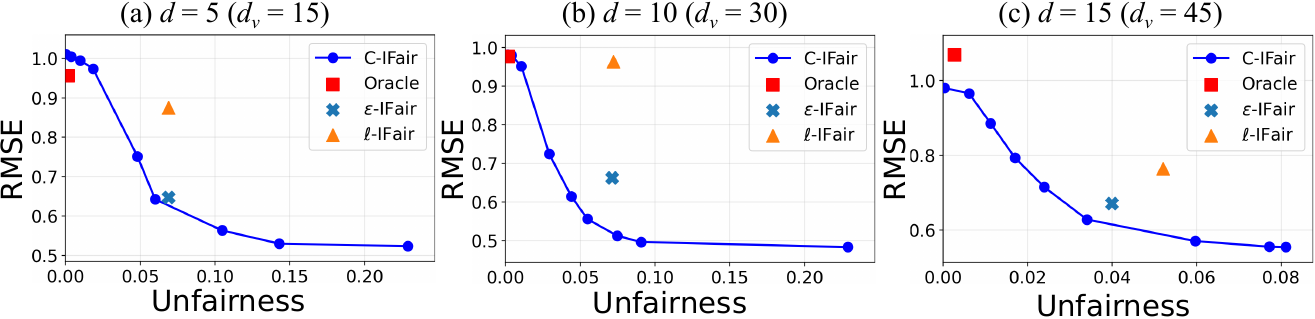}
		\centering 
			\caption{Performance on linear datasets when varying $\lambda$ from $0$ to $200$: 
      (a): $d = 5$, (b): $d = 10$, and (c): $d = 15$ clusters.} 
		\label{fig:tradeoff}
	\end{figure*}

\section{Experiments} \label{sec:experiments}

\paragraph{Baselines.} We compare our cluster interventional fairness (\textbf{C-IFair}) method with six baselines.
(\myra) \textbf{Full} uses all features without fairness constraints.
(\myrb) \textbf{Unaware} uses all features except the sensitive features $\biA$.
(\myrc) \textbf{No-DesCs} uses all features except \textit{definite descendant clusters} of $\biA$ (i.e., clusters that are descendants of $\biA$ in all cluster DAGs represented by the cluster CPDAG).
(\myrd) $\epsilon$\textbf{-IFair} \citep{zuo2024interventional} penalizes the kernel MMD estimated using a variable-level CPDAG.
(\myre) $\ell$\textbf{-IFair} \citep{li2024local} penalizes unfairness based on a local adjustment set search over a variable-level CPDAG.
(\myrf) \textbf{Oracle} assumes access to the true variable-level DAG and uses all non-descendant \textit{variables} of $\biA$.

\subsection{Synthetic Data Experiments} \label{subsec:synthetic_experiments}

\begin{table*}[t]
  \centering
  \caption{Area Under the Curve (AUC) and unfairness for real-world datasets on the held-out test set using 20 random seeds}
  \label{table:real_auc_unfairness}
  \scshape
  \scalebox{0.95}{%
  \begin{tabular}{lcccccc}
    \toprule
    & \multicolumn{2}{c}{\textsc{Adult}} & \multicolumn{2}{c}{\textsc{German}} & \multicolumn{2}{c}{\textsc{OULAD}} \\
    \cmidrule(lr){2-3}\cmidrule(lr){4-5}\cmidrule(lr){6-7}
    & \textsc{AUC}$\uparrow$ & \textsc{Unfairness}$\downarrow$
    & \textsc{AUC}$\uparrow$ & \textsc{Unfairness}$\downarrow$
    & \textsc{AUC}$\uparrow$ & \textsc{Unfairness}$\downarrow$ \\
    \midrule
    \textsc{Oracle} 
      & $0.709 \pm 0.009$ & $0.000 \pm 0.000$
      & $0.582 \pm 0.004$ & $0.000 \pm 0.000$
      & $0.628 \pm 0.015$ & $0.000 \pm 0.000$ \\
    \midrule
    \textsc{Full} 
      & $0.874 \pm 0.004$ & $0.030 \pm 0.024$
      & $0.762 \pm 0.060$ & $0.173 \pm 0.011$
      & $0.697 \pm 0.024$ & $0.061 \pm 0.001$ \\
    \midrule
    \textsc{Unaware} 
      & $0.805 \pm 0.049$ & $0.018 \pm 0.014$
      & $0.695 \pm 0.061$ & $0.101 \pm 0.013$
      & $0.654 \pm 0.019$ & $0.004 \pm 0.000$ \\
    \textsc{No-DesCs} 
      & $0.815 \pm 0.007$ & $0.026 \pm 0.022$
      & $0.726 \pm 0.059$ & $0.085 \pm 0.008$
      & $0.654 \pm 0.020$ & $0.004 \pm 0.000$ \\
    $\epsilon$-\textsc{IFair} 
      & $0.812 \pm 0.006$ & $0.015 \pm 0.012$
      & $0.756 \pm 0.008$ & $0.082 \pm 0.000$
      & $0.641 \pm 0.005$ & $0.005 \pm 0.002$ \\
    $\ell$-\textsc{IFair} 
      & $0.764 \pm 0.002$ & $0.015 \pm 0.008$
      & $0.750 \pm 0.035$ & $0.151 \pm 0.011$
      & $0.639 \pm 0.011$ & $0.003 \pm 0.001$ \\
    \textbf{\textsc{C-IFair}} 
      & $0.829 \pm 0.021$ & $0.014 \pm 0.010$
      & $0.760 \pm 0.061$ & $0.065 \pm 0.008$
      & $0.660 \pm 0.018$ & $0.001 \pm 0.000$ \\
    \bottomrule
  \end{tabular}%
  }
\end{table*}

\paragraph{Settings.} We use linear and nonlinear synthetic datasets generated from variable-level causal DAGs. We randomly sample a causal DAG with $d_v = 3d$ variables (for each $d \in \{5,10,15\}$) from an Erd\H{o}s--R\'enyi (ER) model with an expected degree of 2. Using the resulting DAGs and corresponding linear or nonlinear SCMs, we generate 20 datasets with sample size $n=5000$, each of which is split into training (80\%), validation (10\%), and test (10\%) sets.

By applying CLOC to each dataset, we infer a cluster CPDAG over $d$ clusters, using a randomly generated partition with 3 variables per cluster. We then select one cluster consisting of binary variables as the sensitive cluster $\biA$. 
In our main experiments, we set the admissible cluster $\biX^{\mathrm{ad}}$ to be empty; results with $\biX^{\mathrm{ad}}$ are provided in \Cref{subsec-add-admissible}.

To evaluate unfairness, we estimate the MMD on the test set by sampling interventional datasets with the true SCM.

\paragraph{Performance Comparison.} \Cref{table:synth_rmse_unfairness} reports the mean and standard deviation of RMSE and unfairness over $20$ datasets.

Among the baselines (excluding \textbf{Oracle} and \textbf{Full}), our \textbf{C-IFair} performs best on average in both RMSE and unfairness, highlighting the effectiveness of our cluster-graph-based framework. 
While maintaining high accuracy, it achieves fairness closest to \textbf{Oracle}, which requires the true variable-level DAG—an unrealistic requirement in practice.

\textbf{Unaware} and \textbf{No-DesCs} exhibit high RMSE and unfairness, implying that discarding potentially predictive features is insufficient to remove indirect effects 
due to graph uncertainty.
$\epsilon$\textbf{-IFair} and $\ell$\textbf{-IFair} produce less fair predictions than our \textbf{C-IFair} in high-dimensional setups ($d=15$),
indicating that inferred variable-level CPDAGs can be unreliable in such setups. 
We provide supporting evidence for this claim in \Cref{subsec-add-graph-inference} 
in comparison with cluster CPDAG inference, following the procedure of \citet{anand2025causal}.

\paragraph{Accuracy--Fairness Trade-Off.} To assess the effect of penalty function $g_{\theta}$, we vary the parameter $\lambda$ from $0$ to $200$.

\Cref{fig:tradeoff} shows results on the linear datasets. As expected, increasing $\lambda$ reduces unfairness at the cost of higher RMSE,
illustrating how the choice of penalty parameter $\lambda$ adjusts 
the fairness--accuracy trade-off in \textbf{C-IFair}.

\subsection{Real-World Data Experiments} \label{subsec:real_experiments}

\paragraph{Settings.} We use three standard benchmarks, the Adult, German credit, and OULAD datasets. 
To evaluate unfairness in the absence of ground-truth SCMs, 
we compute the weighted MMD estimator in \eqref{eq-int-mmd} 
by leveraging the variable-level DAGs used in prior works (\Cref{subsec-realdata}).

\paragraph{Performance Comparison.} \Cref{table:real_auc_unfairness} presents the AUC and unfairness on the three datasets.
Again, our \textbf{C-IFair} achieves the best average performance in both AUC and unfairness among all baselines (except \textbf{Oracle} and \textbf{Full}), thus underscoring its practical utility for real-world applications.

\paragraph{Fairness Illustration.}
We compare the inferred interventional distributions over prediction $\hat{Y} \in \{0, 1\}$ on the Adult dataset.
Following \citet{zuo2024interventional}, we estimate probability $\pr(\hat{Y} = 1 \mid do(\biA = \bia))$ for different values $\bia$ by approximating the SCM via conditional density estimation.

\Cref{fig:int-diff} shows the probability difference from baseline value $\bia = [0,0]$.
Our \textbf{C-IFair} makes sufficiently fair predictions with respect to sensitive features $\biA$ (gender and race),
thus highlighting its effectiveness in practical applications.

\subsection{Additional Experiments} \label{subsec:summary_add_experiments}

Results in \Cref{sec-add-experiments} support the following key findings:
\begin{itemize}
  \item \textbf{In presence of admissible clusters:} Our \textbf{C-IFair} performs best in these challenging setups (\Cref{subsec-add-admissible}).
  \item \textbf{Dense graphs:} Despite the change in number of adjustment sets $\npp$, our \textbf{C-IFair} works best (\Cref{subsec-add-dense}).
  \item \textbf{\textit{Inadmissible} partitions:} Empirically, the inferred graphs can still enforce fairness (\Cref{subsec-add-violate}). 
\end{itemize}

\section{Related Work} \label{sec:related_work}

	\begin{figure}[t]
		\includegraphics[height=3.8cm]{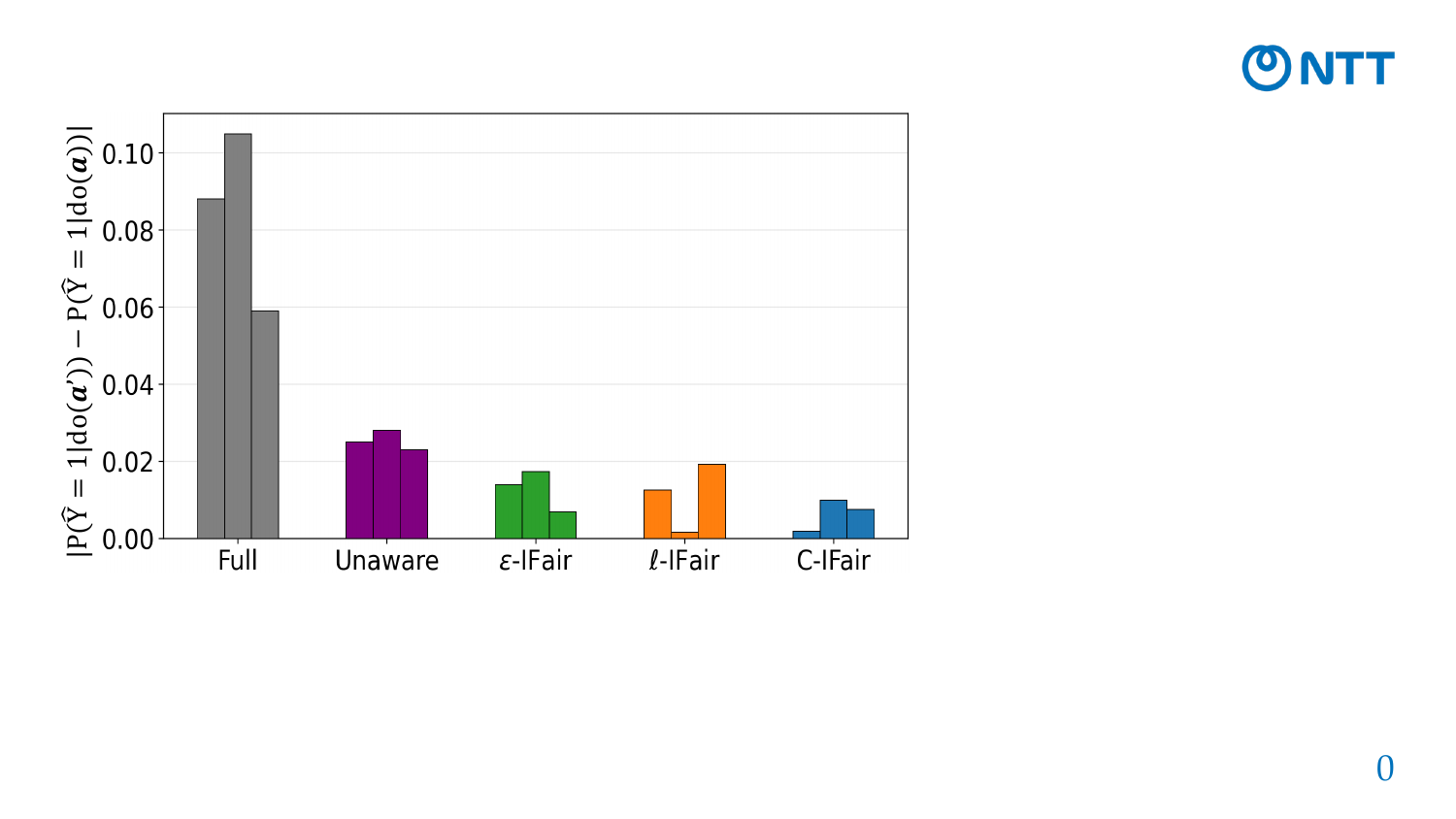}
		\centering 
			\caption{Probability difference between intervened values $\bia'=[0,1], [1,0], [1,1]$ and $\bia=[0,0]$ on Adult dataset} 
		\label{fig:int-diff}
	\end{figure}

  \paragraph{Fairness under Causal Graph Uncertainty.}
Achieving fairness under causal graph uncertainty has been a long-standing challenge. Prior work has addressed counterfactual fairness using manually designed candidate graphs \citep{russell2017worlds,chikahara2023making}. $\epsilon$\textbf{-IFair} \citep{zuo2024interventional} was the first to leverage a CPDAG, but it assumes that no feature is connected to a sensitive feature via an undirected edge. $\ell$\textbf{-IFair} \citep{li2024local} relaxes this assumption, yet it still requires accurate CPDAG estimation.

Our \textbf{C-IFair} advances this line of work by providing a principled framework for leveraging cluster-level causal structure, which can be reliably inferred from the data.

\paragraph{Cluster Causal Graphs.}
While causal discovery has traditionally focused on standard causal graphs over individual variables \citep{spirtes2000causation,shimizu2006linear,hoyer2009nonlinear,zhang2009identifiability,chikahara2018causal,glymour2019review,chikahara2026moment,ong2026metacadi},
there is growing interest in cluster causal graphs, which aim to capture macro-level causal relations among groups of variables.
Early work focused on Bayesian networks over groups of variables, called \textit{module networks} \citep{segal2005learning,parviainen2017learning}. 
Building on the SCM framework, \citet{anand2023causal} introduced cluster DAGs as causal graphs whose nodes are user-specified clusters of variables. 
Subsequent work has studied theoretical foundations and practical discovery methods for such cluster causal graphs \citep{niu2022learning,wahl2023vector,tikka2023clustering,wahl2024foundations,li2026local}.

Our learning framework builds on the CLOC method developed by \citet{anand2025causal}, which learns a cluster CPDAG without imposing restrictive assumptions on the graph structure within each cluster. 
This property is crucial for our goal of achieving interventional fairness under limited causal knowledge. 
A remaining limitation is the admissible partition assumption (Assumption~\ref{asmp:cpdag}), 
which requires the user-specified partition not to induce directed cycles at the cluster level; we discuss this assumption further in \Cref{app:cluster-formation}.
Recent theoretical results on the cluster DAGs have shown the possibility of relaxing this assumption \citep{yvernes2026relaxing}.
Incorporating such promising results into fairness-aware prediction is an important direction for future work.

\section{Conclusion} \label{sec:conclusion}

We propose a learning framework for achieving interventional fairness from the cluster CPDAG structure,
whose inference requires substantially fewer conditional independence tests than the variable-level CPDAG.
To overcome the identification challenges in cluster CPDAGs, 
we establish a graphical algorithm for identifying valid adjustment sets.
By computing the maximum MMD over these sets,
we enforce interventional fairness in the presence of graph uncertainty.

A promising direction for future work is to extend our framework to incorporate additional domain knowledge about the graph structure, just as the extension of CPDAG-based methods to 
maximally oriented partial directed acyclic graphs (MPDAGs) 
\citep{li2024local, zuo2024interventional}. 
Although extending cluster CPDAGs to cluster MPDAGs is highly non-trivial, it would further reduce the uncertainty in causal structures and improve the fairness-accuracy trade-off.



\begin{acknowledgements} 
  This work was supported by JST ACT-X (JPMJAX23CF).
  The author thanks the anonymous reviewers for their constructive comments, which helped improve the clarity and presentation of the paper.
\end{acknowledgements}

\bibliography{main}

\newpage

\onecolumn

\title{Fairness under Graph Uncertainty: 
Achieving Interventional Fairness \\
with Partially Known Causal Graphs over Clusters of Variables\\(Supplementary Material)}
\maketitle

\appendix

\section{Background on Cluster Causal Graphs} \label{sec:background_cluster_causal_graphs}

\subsection{Notions for Variable-Level Graphs} \label{subsec:graphical_notions}

We first present some graphical notions used in this paper.

\paragraph{Variable-level d-separation.} 
A triplet of nodes $\langle V_i, V_k, V_j \rangle$ in a variable-level DAG $\vDAG$ is \textit{active} relative to a set of nodes $\biZ$ if either 
(\myra) $V_k$ is a collider or any of its descendants are in $\biZ$, or (\myrb) 
$V_k$ is a non-collider and $V_k \notin \biZ$. 
A path $p$ is \textit{active} given $\biZ$ if every
triplet on $p$ is active relative to $\biZ$. Otherwise, $p$ is \textit{inactive}.
Given a variable-level causal DAG $\vDAG$, two node sets $\biX$ and $\biY$ are \textit{d-separated} by a set of nodes $\biZ$, if there is no active path between $\biX$ and $\biY$ given $\biZ$.
We denote d-separation in a variable-level causal DAG $\vDAG$ by $(\biX, \biY \mid \biZ)_{\vDAG}$.

\paragraph{Definite Colliders and Definite Non-Colliders in Variable-level CPDAGs.} A variable-level CPDAG has directed and undirected edges, where directed edges are common to all DAGs represented by the CPDAG, while undirected edges indicate uncertainty in edge directions in the DAGs.
A \textit{definite collider} and a \textit{definite non-collider} are a collider and a non-collider in all DAGs represented by the variable-level CPDAG, respectively.
They are determined for each \textit{shielded} triplet $\langle X, Z, Y \rangle$ (i.e., a triplet where nodes $X$ and $Y$ are adjacent to each other) and each \textit{unshielded} triplet $\langle X, Z, Y \rangle$ (i.e., a triplet where nodes $X$ and $Y$ are non-adjacent to each other).
$Z$ in a triplet $\langle X, Z, Y \rangle$ is a definite collider if $X \to Z \leftarrow Y$ in the CPDAG.
$Z$ is a definite non-collider if either
(\myra) at least one edge is out of it (i.e., $X \leftarrow Z - Y$ or $X - Z \rightarrow Y$), or 
(\myrb) both edges are undirected ($X - Z - Y$) and the triplet is unshielded.

\subsection{Definition of Cluster Causal Graphs} \label{subsec:definition_cluster_causal_graphs}

We present a brief overview of the definition of cluster causal graphs based on \citet{anand2025causal}.

\paragraph{Independence arcs.} Independence arcs describe the conditional independence relationships in an unshielded triplet of clusters $\langle \biC_i, \biC_k, \biC_j \rangle$ in a cluster DAG. For a shielded triplet $\langle \biC_i, \biC_k, \biC_j \rangle$, we consider the unshielded triplet formed by removing the edge between $\biC_i$ and $\biC_j$, which is referred to as a \textit{manipulated unshielded triplet} in \citet{anand2025causal}.

\begin{definition}[Independence arcs \citep{anand2025causal}]
For any unshielded triplet (and any manipulated unshielded triplet) of clusters $\langle \biC_i,\biC_k,\biC_j\rangle$ in a cluster DAG $\cDAG$ over clusters $\cNode=\{\biC_1,\dots,\biC_m\}$,
let $S$ be a (possibly empty) set of clusters
\[
S \subseteq \cNode \setminus \{\biC_i,\biC_j\}
\quad \text{such that} \quad
\biC_i \perp \biC_j \mid S,
\]
if at least one such separating set exists (if multiple exist, fix any one of them).

Define the independence arc $\IA_{\langle \biC_i,\biC_k,\biC_j\rangle}$ (drawn between the two edges incident to $\biC_k$ in the triplet) by:
\begin{enumerate}
    \item $\IA_{\langle \biC_i,\biC_k,\biC_j\rangle}=\typeMarg$ \quad if and only if \quad $\biC_k \in S$.
    \item $\IA_{\langle \biC_i,\biC_k,\biC_j\rangle}=\typeCond$ \quad if and only if \quad $\biC_k \notin S$ and $\biC_i \not\perp \biC_j \mid (S \cup \{\biC_k\})$.
    \item $\IA_{\langle \biC_i,\biC_k,\biC_j\rangle}=\typeNever$ \quad if and only if \quad $\biC_k \notin S$ and $\biC_i \perp \biC_j \mid (S \cup \{\biC_k\})$.
\end{enumerate}
\label{def:independence_arc}
\end{definition}

\paragraph{Separation and Connection Marks.} Separation and connection marks annotate independence arcs with $\typeMarg$ and $\typeCond$ and those with $\typeNever$, respectively, to indicate the exceptional cases where the corresponding conditional independence relationships do not hold.

Formally, separation marks are defined as follows.
\begin{definition}[Separation marks \citep{anand2025causal}]
Let $\vDAG$ be a variable-level DAG and $\cDAG$ be a cluster DAG with independence arcs as in \Cref{def:independence_arc}.
Consider a cluster path
\[
p_{\mathrm{c}}=\langle \biC_{1},\biC_{2},\dots,\biC_{n}\rangle
\quad (n\ge 4)
\]
and its arc trajectory
\[
a=\langle \IA_{\langle \biC_{1},\biC_{2},\biC_{3}\rangle},\ \IA_{\langle \biC_{2},\biC_{3},\biC_{4}\rangle},\ \dots,\ \IA_{\langle \biC_{n-2},\biC_{n-1},\biC_{n}\rangle}\rangle.
\]
Assume:
\begin{enumerate}
    \item no arc in $a$ is of type $\typeNever$;
    \item there exists \emph{no} d-connecting variable-level path in $\vDAG$ (relative to some set of clusters) that is analogous to $p_{\mathrm{c}}$;
    \item there exists a d-connecting variable-level path (relative to some set of clusters) analogous to the prefix path $\langle \biC_{1},\dots,\biC_{n-1}\rangle$;
    \item there exists a d-connecting variable-level path (relative to some set of clusters) analogous to the suffix path $\langle \biC_{2},\dots,\biC_{n}\rangle$.
\end{enumerate}
Then we place a \emph{separation mark} $\oslash_{\biC_{1}}$ on the last arc $\IA_{\langle \biC_{n-2},\biC_{n-1},\biC_{n}\rangle}$, and symmetrically place $\oslash_{\biC_{n}}$ on the first arc $\IA_{\langle \biC_{1},\biC_{2},\biC_{3}\rangle}$.
\label{def:separation_mark}
\end{definition}
In \Cref{def:separation_mark}, the concept of \textit{analogous paths} is used, which is defined as follows.
\begin{definition}[Analogous paths \citep{anand2025causal}]
Let $\vDAG$ be a variable-level DAG, and let $\cDAG$ be a cluster DAG over clusters
$\cNode=\{\biC_1,\dots,\biC_d\}$, where each cluster $\biC$ contains a set of variables in $\vDAG$.
Consider a path over clusters
\[
p_{\mathrm{c}}=\langle \biC_{1},\biC_{2},\dots,\biC_{n}\rangle
\]
in $\cDAG$, and a path over variables
\[
p_{\mathrm{v}}=\langle V_{1},V_{2},\dots,V_{m}\rangle
\]
in $\vDAG$.
We say that $p_{\mathrm{v}}$ is \emph{analogous} to $p_{\mathrm{c}}$ (and vice versa) if and only if:
\begin{enumerate}
    \item every variable on $p_{\mathrm{v}}$ belongs to one of the clusters appearing on $p_{\mathrm{c}}$;
    \item every cluster appearing on $p_{\mathrm{c}}$ contributes at least one variable to $p_{\mathrm{v}}$;
    \item the order along $p_{\mathrm{v}}$ is consistent with the order of clusters along $p_{\mathrm{c}}$ in the sense that,
    once $p_{\mathrm{v}}$ visits some variable in the last cluster $\biC_{t}$ (for any $t$), it never later visits a variable that lies in any earlier cluster $\biC_{s}$ with $s<t$ ($s, t \in \{1, \dots, n\}$).
\end{enumerate}
\label{def:analogous_paths}
\end{definition}

By contrast, connection marks are defined as follows.
\begin{definition}[Connection marks \citep{anand2025causal}]
Let $\vDAG$ be a variable-level DAG and $\cDAG$ be a cluster DAG with independence arcs.
Consider a triplet of clusters $\langle \biC_i,\biC_k,\biC_j\rangle$ in $\cDAG$ and its independence arc $\IA_{\langle \biC_i,\biC_k,\biC_j\rangle}$.
Assume that $\IA_{\langle \biC_i,\biC_k,\biC_j\rangle}\in\{\typeNever,\typeCond\}$.

If there exists a variable-level path in $\vDAG$ that passes through variables
\[
\langle V_i,\dots,V_k,\dots,V_j\rangle
\quad \text{with } V_i\in \biC_i,\ V_k\in \biC_k,\ V_j\in \biC_j,
\]
then for each collider variable $V_k' \in \biC_k$ that lies on such a path, collect the clusters that contain \emph{children} (equivalently, descendants via one outgoing step) of that collider along the same path:
\[
D \;=\; \bigcup \bigl\{\, \biC_d \in \cNode \setminus \{\biC_i,\biC_k,\biC_j\} \;:\; \exists\, V_d \in \biC_d \text{ with } V_d \in \mathrm{Ch}(V_k') \,\bigr\}.
\]
We then add the \emph{connection mark} $\oplus_{D}$ to the independence arc $\IA_{\langle \biC_i,\biC_k,\biC_j\rangle}$.
\label{def:connection_mark}
\end{definition}

\paragraph{Cluster Causal DAG with Independence Arcs.} Using the notions of independence arcs, separation marks, and connection marks, cluster causal DAGs with independence arcs are defined as follows.
\begin{definition}[Cluster Causal DAGs \citep{anand2025causal}]
Given a variable-level DAG $\vDAG$ and an admissible partition of clusters $\cNode = \{\biC_1, \dots, \biC_d\}$ of variables in $\vDAG$,
 a \textit{cluster causal DAG with independence arcs} 
 is defined as a DAG $\cDAG = (\cNode, \cDir, \IAs)$ over clusters $\cNode$ such that
 \begin{itemize}
  \item  An edge $\biC_i \rightarrow \biC_j$ is in $\cDir$ if there exist some $V_{i} \in \biC_i$ and $V_{j} \in \biC_j$ such that $V_{i} \rightarrow V_{j}$ is in $\vDAG$; 
  \item The set of independence arcs $\IAs$ is defined according to \Cref{def:independence_arc};
  \item The separation marks and connection marks are defined according to \Cref{def:separation_mark} and \Cref{def:connection_mark}, respectively.
 \end{itemize}
 \label{def:cluster_causal_dag}
\end{definition}

\paragraph{Cluster CPDAGs.} A cluster CPDAG with independence arcs is defined as the equivalence class of cluster causal DAGs with independence arcs that entail the same set of conditional independence relationships over clusters.
\begin{definition}[Cluster CPDAGs \citep{anand2025causal}]
Let $\cMEC$ be the cluster MEC of a cluster causal DAG $\cDAG$ over clusters $\cNode = \{\biC_1, \dots, \biC_d\}$.
The \textit{cluster CPDAG with independence arcs} for the cluster MEC $\cMEC$, denoted by $\cCPDAG$, is defined such that 
\begin{itemize}
\item $\cCPDAG$ has the same skeleton as $\cDAG$ and its any member of $\cMEC$;
\item A directed edge is present in $\cCPDAG$ if it is present in all members of $\cMEC$; otherwise, the edge is undirected;
\item An independence arc is present in $\cCPDAG$ if it is present in all members of $\cMEC$; otherwise, the independence arc is absent;
\item $\cCPDAG$ has the same separation marks and connection marks as $\cDAG$ and its any member of $\cMEC$. 
\end{itemize}  
\label{def:cluster_cpdag}
\end{definition}

\subsection{Assumption on  Cluster Causal Graphs} \label{subsec:assumption_cluster_causal_graphs}

The CLOC algorithm \citep{anand2025causal} for learning cluster CPDAGs from data relies on the faithfulness assumption at the cluster level (Assumption \ref{asmp:faithful}), which is defined as follows.
\begin{definition}[Cluster-level Faithfulness \citep{anand2025causal}] 
Let $\cDAG$ be a cluster DAG $\cDAG$ over clusters $\cNode$, and let $\pr(\cNode)$ be the joint distribution $\pr(\cNode)$ over $\cNode$ that is generated from an SCM compatible with a causal DAG $\vDAG$ compatible with $\cDAG$.
Then $\pr(\cNode)$ is \textit{faithful} to $\cDAG$ if conditional independence relationships across clusters $\cNode$ imply the cluster-level d-separation (\Cref{def:d-sep}) in $\cDAG$; that is, for any disjoint cluster sets $\biX, \biY, \biZ \subseteq \cNode$,
\begin{align*}
  (\biX \indep \biY \mid \biZ)_{\pr(\cNode)} \implies (\biX, \biY \mid \biZ)_{\cDAG}.
\end{align*}
\label{def:faithful}
\end{definition}

\section{Derivations} 

In this section, we provide the detailed derivations of the equations presented in the main paper.

\subsection{Barycenter Kernel MMD (Eq. (7))} \label{dev:barycenter_mmd}

In Eq. (7), we decompose the summation of kernel MMDs over sensitive feature value pairs $(\bia, \bia')$ into the sum of barycenter MMDs. We drop the subscript $\ipp$ for adjustment set $\biZ_{\ipp}$ for notational simplicity and restate Eq. (7) as
\begin{align*}
\sum_{\bia, \bia'} \mathrm{MMD}\left(\pr_{\biYhat| do(\bia), do(\bix^{\mathrm{ad}})}, \pr_{\biYhat| do(\bia'), do(\bix^{\mathrm{ad}})}\right)  
= 2 N_{\biA} \sum_{\bia} \mathrm{MMD}\left(\pr_{\biYhat| do(\bia), do(\bix^{\mathrm{ad}})}, \frac{1}{N_{\biA}} \sum_{\bia'} \pr_{\biYhat| do(\bia'), do(\bix^{\mathrm{ad}})}\right).
\end{align*}

\paragraph{Derivation.}
Fix $\bix^{\mathrm{ad}}$ and define
$P_{\bia} := \pr_{\biYhat\mid do(\bia),\, do(\bix^{\mathrm{ad}})}$.
Let $\mathcal{H}$ be the RKHS associated with the kernel and feature map
$\phi:\mathrm{supp}(\biYhat)\to \mathcal{H}$, and define the kernel mean embedding
\[
\mu_{\bia} \;:=\; \E_{\biYhat\sim P_{\bia}}\!\left[\phi(\biYhat)\right]\in\mathcal{H}.
\]
We use the (squared) RKHS distance representation of the kernel MMD:
\[
\mathrm{MMD}(P,Q)\;=\;\|\mu_P-\mu_Q\|_{\mathcal{H}}^{2}.
\]
Then
\[
\sum_{\bia,\bia'} \mathrm{MMD}(P_{\bia},P_{\bia'})
=
\sum_{\bia,\bia'} \|\mu_{\bia}-\mu_{\bia'}\|_{\mathcal{H}}^{2}.
\]
Expanding the squared norm gives
\begin{align*}
\sum_{\bia,\bia'} \|\mu_{\bia}-\mu_{\bia'}\|_{\mathcal{H}}^{2}
=
\sum_{\bia,\bia'}\Big(\|\mu_{\bia}\|_{\mathcal{H}}^{2}
+\|\mu_{\bia'}\|_{\mathcal{H}}^{2}
-2\langle \mu_{\bia},\mu_{\bia'}\rangle_{\mathcal{H}}\Big) 
=
2N_{\biA}\sum_{\bia}\|\mu_{\bia}\|_{\mathcal{H}}^{2}
-2\Big\|\sum_{\bia}\mu_{\bia}\Big\|_{\mathcal{H}}^{2}.
\end{align*}
Now define the barycenter (mixture) distribution
\[
\bar P \;:=\;\frac{1}{N_{\biA}}\sum_{\bia'} P_{\bia'}.
\]
By linearity of expectation, its kernel mean embedding is
\[
\bar\mu \;:=\;\mu_{\bar P}
=\E_{\biYhat\sim \bar P}[\phi(\biYhat)]
=\frac{1}{N_{\biA}}\sum_{\bia'}\mu_{\bia'}.
\]
Expanding the squared norm with $\bar\mu$ yields
\begin{align*}
\sum_{\bia}\|\mu_{\bia}-\bar\mu\|_{\mathcal H}^{2}
=\sum_{\bia}\Big(\|\mu_{\bia}\|_{\mathcal H}^{2}
-2\langle \mu_{\bia},\bar\mu\rangle_{\mathcal H}
+\|\bar\mu\|_{\mathcal H}^{2}\Big)
=\sum_{\bia}\|\mu_{\bia}\|_{\mathcal H}^{2}
-2\left\langle \sum_{\bia}\mu_{\bia},\bar\mu\right\rangle_{\mathcal H}
+N_{\biA}\|\bar\mu\|_{\mathcal H}^{2}.
\end{align*}
Since the inner product is given by
\begin{align*}
\left\langle \sum_{\bia}\mu_{\bia},\bar\mu\right\rangle_{\mathcal H}
=
\left\langle \sum_{\bia}\mu_{\bia},\frac1{N_{\biA}}\sum_{\bia'}\mu_{\bia'}\right\rangle_{\mathcal H}
=
\frac1{N_{\biA}}\left\|\sum_{\bia}\mu_{\bia}\right\|_{\mathcal H}^{2}
=
N_{\biA}\|\bar\mu\|_{\mathcal H}^{2},
\end{align*}
the sum of squared norms simplifies to
\begin{align*}
\sum_{\bia}\|\mu_{\bia}-\bar\mu\|_{\mathcal H}^{2}
=
\sum_{\bia}\|\mu_{\bia}\|_{\mathcal H}^{2}
-2N_{\biA}\|\bar\mu\|_{\mathcal H}^{2}
+N_{\biA}\|\bar\mu\|_{\mathcal H}^{2}
=
\sum_{\bia}\|\mu_{\bia}\|_{\mathcal H}^{2}
- N_{\biA}\|\bar\mu\|_{\mathcal H}^{2}.
\end{align*}
Hence we have
\begin{align*}
\sum_{\bia,\bia'} \|\mu_{\bia}-\mu_{\bia'}\|_{\mathcal{H}}^{2}
=
2N_{\biA}\sum_{\bia}\|\mu_{\bia}\|_{\mathcal{H}}^{2}
-2N_{\biA}^{2}\|\bar\mu\|_{\mathcal{H}}^{2} 
=
2N_{\biA}\sum_{\bia}\|\mu_{\bia}-\bar\mu\|_{\mathcal{H}}^{2},
\end{align*}
Finally, since $\|\mu_{\bia}-\bar\mu\|_{\mathcal{H}}^{2}=\mathrm{MMD}(P_{\bia},\bar P)$, we obtain
\[
\sum_{\bia,\bia'} \mathrm{MMD}(P_{\bia},P_{\bia'})
=
2N_{\biA}\sum_{\bia}\mathrm{MMD}\!\left(
P_{\bia},\frac{1}{N_{\biA}}\sum_{\bia'}P_{\bia'}
\right),
\]
which is Eq.~(7).

\subsection{Identification and Estimation of Interventional Distributions (Eqs. (5) and (10))} \label{dev:adjustment_formula}

We first restate Eqs. (5) and (10) below.

\begin{align*}
    \begin{aligned}
    \pr(\biYhat = \biyhat \mid do(\biA = \bia), do(\biX^{\mathrm{ad}} = \bix^{\mathrm{ad}}))  
    &= \E_{\biZ, \biX^{\mathrm{re}} \mid \bia, \bix^{\mathrm{ad}}}\left[\pr(\biYhat = \biyhat | \biA = \bia, \biX^{\mathrm{ad}} = \bix^{\mathrm{ad}}, \biZ, \biX^{\mathrm{re}}) \right] \\
    &= \E\left[\frac{\I(\bia, \bix^{\mathrm{ad}})}{\pi_{\bia, \bix^{\mathrm{ad}}}(\biZ)} \E_{\biX^{\mathrm{re}} \mid \biA, \biX^{\mathrm{ad}}, \biZ} \left[ \pr(\biYhat = \biyhat \mid \biA,  \biX^{\mathrm{ad}}, \biZ, \biX^{\mathrm{re}}) \right] \right].
    \end{aligned} 
\end{align*}

Below we provide the derivation of these equations.

\paragraph{Derivation.}
We derive Eq.~\eqref{eq-adjustment} and Eq.~\eqref{eq-IPW} from the observed distribution.
Let $\biZ$ be a valid adjustment set, and let $\biX^{\mathrm{re}}$ denote the remaining features except $\biA$, $\biX^{\mathrm{ad}}$, and $\biZ$.

\smallskip
\noindent\textbf{Step 1 (Eq.~\eqref{eq-adjustment}).}
By marginalizing out $\biZ$ and $\biX^{\mathrm{re}}$ under the intervention, we obtain
\begin{align}
&\pr(\biYhat=\biyhat \mid do(\biA=\bia),do(\biX^{\mathrm{ad}}=\bix^{\mathrm{ad}})) \nonumber \\
= &\sum_{\biz}\sum_{\bix^{\mathrm{re}}}
\pr(\biYhat=\biyhat\mid \biA=\bia,\biX^{\mathrm{ad}}=\bix^{\mathrm{ad}},\biZ=\biz,\biX^{\mathrm{re}}=\bix^{\mathrm{re}})\,
\pr(\biX^{\mathrm{re}}=\bix^{\mathrm{re}}\mid \biA=\bia,\biX^{\mathrm{ad}}=\bix^{\mathrm{ad}},\biZ=\biz)\,
\pr(\biZ=\biz) \nonumber\\
=
&\E_{\biZ}\!\left[
\E_{\biX^{\mathrm{re}}\mid \biA=\bia,\biX^{\mathrm{ad}}=\bix^{\mathrm{ad}},\biZ}\!
\Bigl[\pr(\biYhat=\biyhat\mid \biA=\bia,\biX^{\mathrm{ad}}=\bix^{\mathrm{ad}},\biZ,\biX^{\mathrm{re}})\Bigr]
\right]. \label{eq:dev_eq5}
\end{align}

\smallskip
\noindent\textbf{Step 2 (Eq.~\eqref{eq-IPW}).}
Define the joint propensity score by
\[
\pi_{\bia,\bix^{\mathrm{ad}}}(\biZ)\coloneqq \pr(\biA=\bia,\biX^{\mathrm{ad}}=\bix^{\mathrm{ad}}\mid \biZ),
\]
and assume $\pi_{\bia,\bix^{\mathrm{ad}}}(\biZ)>0$ almost surely.

\begin{lemma}[IPW identity with joint propensity score]
\label{lem:ipw_identity}
For any integrable function $H(\biA,\biX^{\mathrm{ad}},\biZ,\biX^{\mathrm{re}})$,
\begin{align}
\E\!\left[
\frac{\I(\biA=\bia,\biX^{\mathrm{ad}}=\bix^{\mathrm{ad}})}{\pi_{\bia,\bix^{\mathrm{ad}}}(\biZ)}
\;\E_{\biX^{\mathrm{re}}\mid \biA,\biX^{\mathrm{ad}},\biZ}\!\left[H(\biA,\biX^{\mathrm{ad}},\biZ,\biX^{\mathrm{re}})\right]
\right]
=
\E_{\biZ}\!\left[
\E_{\biX^{\mathrm{re}}\mid \biA=\bia,\biX^{\mathrm{ad}}=\bix^{\mathrm{ad}},\biZ}\!\left[H(\bia,\bix^{\mathrm{ad}},\biZ,\biX^{\mathrm{re}})\right]
\right]. \label{eq:ipw_identity}
\end{align}
\end{lemma}

\begin{proof}
By the law of iterated expectations,
\begin{align*}
\E\!\left[
\frac{\I(\biA=\bia,\biX^{\mathrm{ad}}=\bix^{\mathrm{ad}})}{\pi_{\bia,\bix^{\mathrm{ad}}}(\biZ)}
\;\E_{\biX^{\mathrm{re}}\mid \biA,\biX^{\mathrm{ad}},\biZ}\!\left[H\right]
\right] 
=
&\E_{\biZ}\!\left[
\E\!\left[
\left.
\frac{\I(\biA=\bia,\biX^{\mathrm{ad}}=\bix^{\mathrm{ad}})}{\pi_{\bia,\bix^{\mathrm{ad}}}(\biZ)}
\;\E_{\biX^{\mathrm{re}}\mid \biA,\biX^{\mathrm{ad}},\biZ}\!\left[H\right]
\;\right|\;\biZ
\right]\right] \\
=
&\E_{\biZ}\!\left[
\frac{1}{\pi_{\bia,\bix^{\mathrm{ad}}}(\biZ)}
\E\!\left[
\left.
\I(\biA=\bia,\biX^{\mathrm{ad}}=\bix^{\mathrm{ad}})
\;\E_{\biX^{\mathrm{re}}\mid \biA,\biX^{\mathrm{ad}},\biZ}\!\left[H\right]
\;\right|\;\biZ
\right]\right].
\end{align*}
Conditioned on $\biZ$, the indicator selects the event $(\biA,\biX^{\mathrm{ad}})=(\bia,\bix^{\mathrm{ad}})$, hence
\begin{align*}
\E\!\left[
\left.
\I(\biA=\bia,\biX^{\mathrm{ad}}=\bix^{\mathrm{ad}})
\;\E_{\biX^{\mathrm{re}}\mid \biA,\biX^{\mathrm{ad}},\biZ}\!\left[H\right]
\;\right|\;\biZ
\right] 
=
\pr(\biA=\bia,\biX^{\mathrm{ad}}=\bix^{\mathrm{ad}}\mid \biZ)\;
\E_{\biX^{\mathrm{re}}\mid \biA=\bia,\biX^{\mathrm{ad}}=\bix^{\mathrm{ad}},\biZ}\!\left[H(\bia,\bix^{\mathrm{ad}},\biZ,\biX^{\mathrm{re}})\right].
\end{align*}
Since $\pr(\biA=\bia,\biX^{\mathrm{ad}}=\bix^{\mathrm{ad}}\mid \biZ)=\pi_{\bia,\bix^{\mathrm{ad}}}(\biZ)$, the factor cancels out, yielding \eqref{eq:ipw_identity}.
\end{proof}

By applying Lemma~\ref{lem:ipw_identity} to
$H(\biA,\biX^{\mathrm{ad}},\biZ,\biX^{\mathrm{re}})
=\pr(\biYhat=\biyhat\mid \biA,\biX^{\mathrm{ad}},\biZ,\biX^{\mathrm{re}})$,
the right-hand side of \eqref{eq:dev_eq5} is rewritten as
\begin{align}
\E\left[\frac{\I(\biA=\bia,\biX^{\mathrm{ad}}=\bix^{\mathrm{ad}})}{\pi_{\bia, \bix^{\mathrm{ad}}}(\biZ)}
\E_{\biX^{\mathrm{re}} \mid \biA, \biX^{\mathrm{ad}}, \biZ}
\left[
\pr(\biYhat = \biyhat \mid \biA,  \biX^{\mathrm{ad}}, \biZ, \biX^{\mathrm{re}})
\right]\right],
\end{align}
which is Eq.~\eqref{eq-IPW}.

\section{Proofs} \label{proofs}

\subsection{Theorem 3.5} \label{proof:possible_parents}

\pptheorem*

\subsubsection{Lemmas}

We first present two lemmas that will be used in the proof of Theorem 3.5.

\begin{lemma}[Marg forbids an unshielded collider]
Assume $\biU-\biA-\biV$ is an unshielded triple (i.e., $\biU$ and $\biV$ are non-adjacent)
and $\IA_{\langle \biU,\biA,\biV\rangle}=\typeMarg$.
Then no cluster DAG consistent with $\cCPDAG$ can contain the collider
$\biU\to \biA \leftarrow \biV$.
\label{lemma1}
\end{lemma}

\begin{proof}
  From \citet[Remark 1]{anand2025causal}, marginally-connecting independence arc  $\typeMarg$ always implies a non-collider structure. Hence, the triplet $\biU-\biA-\biV$ must form
a non-collider structure at the middle node $\biA$, and the collider structure $\biU\to \biA \leftarrow \biV$
is impossible.
\end{proof}

\begin{lemma}[A back-path forces a directed cycle]
Let $\cDAG$ be any cluster DAG that contains $\cDAG_{\rightarrow}$ as a subgraph.
If $\biU\to\biA$ and $\biA\to\biV$ are edges in $\cDAG$ and there exists a directed path
$\biV\to\cdots\to \biU$ in $\cDAG_{\rightarrow}$, then $\cDAG$ contains a directed cycle.
\label{lemma2}
\end{lemma}

\begin{proof}
All directed edges in $\cDAG_{\rightarrow}$ appear in $\cDAG$, hence the path
$\biV\to\cdots\to \biU$ exists in $\cDAG$. Together with $\biA\to\biV$ and $\biU\to\biA$,
we obtain the directed cycle $\biA\to\biV\to\cdots\to\biU\to\biA$.
\end{proof}

\subsubsection{Proof}

\begin{proof}[Proof of Theorem]
\emph{($\Rightarrow$)} Assume that $S$ is a possible parent set of $\biA$.
Then there exists a cluster DAG $G$ in the cluster MEC represented by $\cCPDAG$
such that $\biU\to\biA$ for all $\biU\in S$ and $\biA\to\biV$ for all
$\biV\in \mathrm{sib}(\biA,\cCPDAG)\setminus S$.

(1) Take distinct $\biU,\biV\in S$. If they are adjacent, we are done.
If they are non-adjacent, then $\biU\to\biA\leftarrow \biV$ is an unshielded collider in $G$.
By \Cref{lemma1}, we cannot have $\IA_{\langle\biU,\biA,\biV\rangle}=\typeMarg$.
Moreover, $\IA_{\langle\biU,\biA,\biV\rangle}=\typeCond$ would compel the collider
in every DAG in the MEC, implying both edges into $\biA$ are directed in $\cCPDAG$,
contradicting that $\biU,\biV$ are siblings (connected to $\biA$ by undirected edges).
Hence $\IA_{\langle\biU,\biA,\biV\rangle}=\typeNever$.

(2) Take $\biU\in S$ and $\biV\in \mathrm{sib}(\biA,\cCPDAG)\setminus S$.
If there were a directed path $\biV\to\cdots\to\biU$ in $\cCPDAG_{\rightarrow}$,
then \Cref{lemma2} would yield a directed cycle in $G$, contradicting acyclicity.
Thus no such path exists.

\emph{($\Leftarrow$)} Assume Conditions (1)--(2).
Orient each undirected edge incident to $\biA$ as $\biU\to\biA$ for $\biU\in S$
and $\biA\to\biV$ for $\biV\in \mathrm{sib}(\biA,\cCPDAG)\setminus S$.
Condition (1) ensures that any unshielded collider created at $\biA$ is allowed
(i.e., it occurs only when the corresponding independence arc is $\typeNever$),
so no independence-arc constraint at $\biA$ is violated.
Condition (2) rules out a back-path in $\cCPDAG_{\rightarrow}$ that would force a directed cycle
once we impose $\biU\to\biA$ and $\biA\to\biV$; see \Cref{lemma2}.

Since $\cCPDAG$ represents a nonempty cluster MEC and our imposed orientations do not contradict
the fixed constraints on independence-arc and acyclicity, there exists
at least one cluster DAG extension $G$ in the MEC realizing these orientations.
In that $G$, the parents of $\biA$ among its siblings are exactly $S$.
Hence $S$ is a possible parent set of $\biA$.
\end{proof}

\subsection{Theorem 3.6} \label{proof:no_connection_marked_parents}

\nocmptheorem*

\begin{proof}[Proof of Theorem]
Let $\cDAG$ be any cluster DAG in the cluster MEC, and let $\biZ$ be the parent set of $\biA$ in $\cDAG$.
Consider any back-door path from $\biA$ in $\cDAG$.
By definition, this path starts with an arrow into $\biA$; hence its first edge has the form
$\biA \leftarrow \biP$, where $\biP$ is a parent of $\biA$.
Since $\biZ$ is the parent set of $\biA$ in $\cDAG$, we have $\biP \in \biZ$.

Let $\biQ$ be the next cluster on this path.
Then the first triplet on the path is $\langle \biA,\biP,\biQ\rangle$.
Since the edge between $\biP$ and $\biA$ is $\biP\to \biA$, the middle cluster $\biP$ is a non-collider on this path.
By Lemma~1 of \citet{anand2025causal}, a conditionally connecting independence arc implies a collider structure.
Therefore, the independence arc $\IA_{\langle \biA,\biP,\biQ\rangle}$ cannot be $\typeCond$.
Thus, it is either $\typeMarg$ or $\typeNever$.

If $\IA_{\langle \biA,\biP,\biQ\rangle}=\typeMarg$, then the path is blocked at this first triplet by Condition~1(a) of \Cref{def:d-sep}, because $\biP\in\biZ$.
If $\IA_{\langle \biA,\biP,\biQ\rangle}=\typeNever$, then the path is blocked at this first triplet by Condition~3 of \Cref{def:d-sep}, because no cluster in $\biZ$ is annotated by a connection mark in any independence arc by assumption.
Hence every back-door path from $\biA$ is blocked by $\biZ$.
Therefore, $\biZ$ d-separates every back-door path from $\biA$ in $\cDAG$.
\end{proof}

\subsection{Theorem 3.7} \label{proof:connection_marked_parents}

\cmptheorem*

\begin{proof}[Proof of Theorem]
Fix any cluster DAG $\cDAG$ in the cluster MEC that realizes the parent set $\biZ^{\ipp}$.
We show that the returned set $\biZ$ d-separates every back-door path from $\biA$ in $\cDAG$.

Consider any back-door path from $\biA$.
Let $\biP$ be the first cluster after $\biA$ on this path, so that the first edge has the form
$\biA\leftarrow \biP$.
Since $\cDAG$ realizes the parent set $\biZ^{\ipp}$, we have $\biP\in \biZ^{\ipp}\subseteq \biZ$.

If $\biP$ does not appear in any connection mark, then the same argument as in the proof of Theorem~3.6 shows that the path is blocked at its first triplet.
Indeed, for the next cluster $\biQ$ on the path, the triplet $\langle \biA,\biP,\biQ\rangle$ has middle cluster $\biP$ as a non-collider.
By Lemma~1 of \citet{anand2025causal}, a $\typeCond$ arc implies a collider structure, so this arc cannot be $\typeCond$.
If it is $\typeMarg$, Condition~1(a) of \Cref{def:d-sep} blocks the path because $\biP\in\biZ$.
If it is $\typeNever$, Condition~3 of \Cref{def:d-sep} blocks the path because $\biP$ is not connection-marked.
Thus, it remains to consider back-door paths whose first parent is connection-marked.

Algorithm~\ref{alg:add-adj-cluster-set} initializes the queue with the pairs adjacent to such connection-marked clusters in $\biZ$.
Therefore, the first relevant triplet on any remaining back-door path is inspected by the queue-based procedure.
Whenever a $\typeNever$ triplet has a connection mark whose marked cluster is included in the current $\biZ$, the algorithm pushes the next pair into the queue.
Thus, by induction along the path, the algorithm continues to inspect subsequent triplets exactly in the cases where the current triplet does not by itself certify blocking under \Cref{def:d-sep}.

Now consider any triplet $\langle \biP,\biQ,\biR\rangle$ inspected by this procedure.

If $\IA_{\langle \biP,\biQ,\biR\rangle}=\typeMarg$, then Algorithm~\ref{alg:add-adj-cluster-set} adds $\biQ$ to $\biZ$.
Hence the final returned set contains $\biQ$, and the path is blocked at this triplet by Condition~1(a) of \Cref{def:d-sep}.
If a separation mark on this arc applies, Condition~1(b) also certifies blocking.

If $\IA_{\langle \biP,\biQ,\biR\rangle}=\typeCond$, then Condition~\ref{item:cond} of \Cref{def:d-sep} requires that neither $\biQ$ nor its true descendants are included in $\biZ$, unless a separation mark certifies separation.
The set $\textsc{PossDesc}(\cCPDAG,\biQ)$ used in the algorithm conservatively contains $\biQ$ and all possible descendants of $\biQ$ in the cluster MEC.
Therefore, every true descendant of $\biQ$ in $\cDAG$ is contained in $\textsc{PossDesc}(\cCPDAG,\biQ)$.
Since the algorithm returns $(\biZ,\mathrm{OK})$, line~19 and the final certification step do not return $\mathrm{Unidentifiable}$.
Thus, for every inspected $\typeCond$ triplet, either
$\biZ\cap \textsc{PossDesc}(\cCPDAG,\biQ)=\emptyset$,
or a separation mark on the arc certifies separation.
In the former case, neither $\biQ$ nor any true descendant of $\biQ$ is conditioned on, so Conditions~2(a) and 2(b) of \Cref{def:d-sep} hold.
In the latter case, Condition~2(c) of \Cref{def:d-sep} holds.

If $\IA_{\langle \biP,\biQ,\biR\rangle}=\typeNever$, then Condition~3 of \Cref{def:d-sep} blocks the path whenever no cluster appearing in a connection mark on this arc is included in $\biZ$.
If some cluster appearing in such a connection mark is included in $\biZ$, then this triplet alone does not certify blocking, and the algorithm pushes the next pair into the queue.
Thus the queue-based propagation continues until it reaches a triplet certified by one of the blocking clauses above, or otherwise returns $\mathrm{Unidentifiable}$.

Since Algorithm~\ref{alg:add-adj-cluster-set} returns $(\biZ,\mathrm{OK})$, the latter case never occurs.
Hence every back-door path from $\biA$ contains a triplet satisfying one of the blocking clauses in \Cref{def:d-sep}.
Therefore, $\biZ$ d-separates every back-door path from $\biA$ in $\cDAG$.
Because $\cDAG$ was arbitrary among the cluster DAGs in the cluster MEC that realize the parent set $\biZ^{\ipp}$, the claim follows.
\end{proof}

\section{Experimental Settings} \label{sec-exp-setup}

\subsection{Baselines} \label{subsec-baselines}

In our experiments, we compare the performance of our method (\textbf{C-IFair}) with six baselines:
\begin{itemize}[leftmargin=0.5cm]
  \item \textbf{Full}, which trains the prediction model with the full feature set without any fairness constraint.
  \item \textbf{Unaware}, which trains the prediction model using all features except the sensitive features $\biA$.
  \item \textbf{No-DesCs}, which trains the prediction model using all features except definite descendant clusters of sensitive features $\biA$. To choose such clusters, we first apply the CLOC algorithm \citep{anand2025causal} to learn the cluster CPDAG from data, and then identify the definite descendant clusters of $\biA$ as the clusters that can be reached from $\biA$ via directed paths.  
  \item $\epsilon$\textbf{-IFair} \citep{zuo2024interventional}, which aims to reduce the kernel MMD between the interventional distributions that are estimated based on the decomposition formula given by the partial causal ordering (PCO) algorithm \citep{perkovic2020identifying}. We employ the publicly available implementation downloaded from \url{https://github.com/aoqiz/Interventional-Fairness-with-Partially-Known-Causal-Graphs}.
  \item $\ell$\textbf{-IFair} \citep{li2024local}, which estimates a conditional average treatment effect (CATE) of $\biA$ on prediction $\biYhat$ for all observed values of features $\biX$ and then penalizes their maximum to ensure interventional fairness. These CATE estimates are given by multiple adjustment sets, each obtained by the local search over a CPDAG. We use the source codes on \url{https://github.com/haoxuanli-pku/NeurIPS24-Interventional-Fairness-with-PDAGs}.
  \item \textbf{Oracle}, which trains the prediction model using all features except the non-descendant variables of sensitive features $\biA$. This method serves as an oracle baseline, as it chooses such input features based on the true variable-level causal graph, which is unavailable in practice.
\end{itemize}

\subsection{Settings of Each Method} \label{subsec-setup}

\begin{itemize}[leftmargin=0.5cm]
  \item \textbf{Model choice:} As the prediction model for each method, we use a two-layer multi-layer perceptron (MLP) with rectified linear unit (ReLU) activation and $32$ hidden units  for all datasets. For our \textbf{C-IFair} method, we parameterize the propensity score model $\pi_{\bia, \bix^{\mathrm{ad}}}(\biZ) \coloneqq \pr(\biA=\bia, \biX^{\mathrm{ad}} = \bix^{\mathrm{ad}} \mid \biZ)$ using a two-layer MLP with ReLU activation and $64$ hidden units. In all experiments, we obtain IPW weights by applying self-normalization, which divides each IPW weight by its empirical average; in \Cref{subsec-add-sensitivity}, we present sensitivity analyses of our method to weight clipping.
  \item \textbf{Hyperparameters:} 
  \begin{itemize}
  \item \textbf{Training settings for prediction models:}   For all methods, we set the batch size to $256$, the number of training epochs to $1000$,  the learning rate to $\mathrm{lr} = 10^{-3}$. 
  \item  \textbf{Optimization algorithm:} For our \textbf{C-IFair} method, we can use any gradient-based stochastic optimization algorithm. In our experiments, we use the AdamW algorithm \citep{loshchilovdecoupled} in PyTorch.
  \item \textbf{Kernel choice:} For our \textbf{C-IFair} method, we use the Gaussian kernel with multiple bandwidths $c_{\gamma} \gamma$, where $c_{\gamma} \in [0.5, 1.0, 2.0, 4.0, 8.0, 16.0]$, and $\gamma$ is selected using the standard heuristic called the \textit{median heuristic} \citep{gretton2012kernel}. We set the number of RFFs to $d_{\mathrm{RFF}} = 128$.
  \item \textbf{Mellowmax temperature:} Our \textbf{C-IFair} method approximately takes the maximum of the MMDs over adjustment candidate sets using the Mellowmax function \citep{asadi2017alternative}.
    We fix its temperature parameter $\omega$ to $\omega = 10.0$ for all experiments. We provide sensitivity analyses of our method to the choice of $\omega$ in \Cref{subsec-add-sensitivity}.
  \item \textbf{Unfairness penalty: } 
  The unfairness penalty parameter cannot be trained on the training data, as it determines the desired trade-off between accuracy and fairness. 
  For our \textbf{C-IFair}, we select the value of parameter $\lambda$ from $\{0., 2.0, \dots, 20.0\}$ based on the validation set;
  in the accuracy-fairness trade-off curve (\Cref{fig:tradeoff}), we vary $\lambda$ in $[0, 1, 2, 5, 10, 20, 50, 100, 200]$.
  For $\epsilon$\textbf{-IFair} and $\ell$\textbf{-IFair}, we choose their penalty parameter values by following the suggestions by \citet{zuo2024interventional} and \citet{li2024local}.
  Regarding $\epsilon$\textbf{-IFair}, we used the search space commonly used in their paper, which is $\{0, 0.5, 5, 20, 60, 100\}$. For $\ell$\textbf{-IFair}, we used the default value in their official implementation.
  \end{itemize}
  \item \textbf{CPDAG inference:} To infer a cluster CPDAG for our \textbf{C-IFair} method, we apply the CLOC algorithm \citep{anand2025causal} by employing the Fisher-z's test in the Python package \texttt{causal-learn} \citep{zheng2024causal} with significance level $\alpha = 0.01$. 
  To learn a variable-level CPDAG for $\ell$\textbf{-IFair}, we run the Peter-Clark (PC) algorithm \citep{spirtes2000causation} implemented in \texttt{causal-learn} \citep{zheng2024causal}
  using the Fisher-z' test with $\alpha = 0.01$. 
  Regarding $\epsilon$\textbf{-IFair}, following the official implementation with default settings, we run the greedy equivalence search (GES) algorithm \citep{chickering2002learning} implemented in a general causal discovery software named TETRAD \citep{ramsey2018tetrad}.
\end{itemize}

\subsection{Synthetic Data Generation Processes} \label{subsec-synthdata}

This section describes the generation processes of the synthetic datasets used in \Cref{subsec:synthetic_experiments}.

\subsubsection{Causal Graph Generation and Cluster Partition} \label{subsubsec-causal-graph}

We randomly generate the ground truth variable-level causal DAGs $\vDAG$ by sampling the ER model with expected node degree $2$.
To ensure acyclicity, we first sample a random topological ordering of variables and generate directed edges only along this ordering.

To obtain a cluster DAG $\cDAG$, we follow its definition (\Cref{subsec:cluster_causal_graphs}): we draw an edge $\biC_i \to \biC_j$ in the cluster DAG $\cDAG$ if there exists at least one edge $V_i \to V_j$ in the variable-level DAG $\vDAG$ such that $V_i \in \biC_i$ and $V_j \in \biC_j$.
To construct a partition of clusters $\cNode = \{\biC_1, \dots, \biC_d\}$, we assign $3$ variable nodes in $\vDAG$ to each cluster.

We randomly choose the sensitive cluster $\biA$ from the clusters with degree greater than or equal to $2$.
In our experiments with admissible features (\Cref{subsec-add-admissible}), we randomly choose admissible features $\biX^{\mathrm{ad}}$ from the children of the sensitive cluster $\biA$ in the cluster DAG $\cDAG$.
If the initially sampled graph has no such child cluster, we add a valid variable-level edge from $\biA$ to a non-outcome cluster while preserving the variable-level topological ordering.

We treat $\biA$ and $\biX^{\mathrm{ad}}$ as binary clusters and $\biY$ as a continuous cluster.
For the remaining clusters, we randomly choose binary and continuous clusters.

\subsubsection{Linear Datasets} \label{subsubsec-lin-data}

Using the random causal graphs described in \Cref{subsubsec-causal-graph}, we generate linear datasets as follows.

To generate the values of features $\biX$, 
we sample them using a linear SCM associated with the variable-level DAG $\vDAG$.
For each continuous variable $X_v$, we sample its value from the following linear structural equation:
\[
X_{v} \;=\; \sum_{u\in \mathrm{pa}(v)} w_{v,u}\, X_{u} \;+\; \varepsilon_{v},
\]
where the noise is sampled from the standard Gaussian distribution $\varepsilon_{v}\sim\mathcal{N}(0,1)$, and
each linear coefficient $w_{v, u}$ are given by $w_{v,u} = s_{v,u}\, m_{v,u}$, where 
\[
s_{v,u}\sim\mathrm{Unif}\{-1,+1\},
\qquad
m_{v,u}\sim\mathrm{Unif}(0.5,\,2.0).
\]
For each binary variable $X_v$, we sample its value from the following structural equation:
\[
X_{v} \sim \mathrm{Bernoulli}\!\left(\sigma\!\left(\sum_{u\in \mathrm{pa}(v)} w_{v,u}\, X_{u}\right)\right),
\qquad
\sigma(t)=\frac{1}{1+e^{-t}}.
\]

To generate the values of target variable $\biY$, we use the following linear structural equation:
\[
Y_v \;=\; \sum_{u \in \mathrm{pa}(v)} w_{v, u} X_u\;+\; \varepsilon_{v},
\]
where the noise is sampled from the standard Gaussian distribution $\varepsilon_{v}\sim\mathcal{N}(0,1)$, 
the parent set $\mathrm{pa}(v)$ is forced to include all feature clusters, 
and each linear coefficient $w_{v, u}$ is given by $w_{v,u} = s_{v,u}\, m_{v,u}$ when $u$ is not sensitive cluster $\biA$; otherwise, we set $w_{v,u} = 5 s_{v,u}\, m_{v,u}$ to introduce unfair bias.

\subsubsection{Nonlinear Datasets} \label{subsubsec-nonlin-data}

We generate nonlinear datasets in the same way as linear dataset generation processes described in \Cref{subsubsec-lin-data}, except that we apply nonlinear function $\xi(\cdot)$ for all variables that have parents in the variable-level DAG $\vDAG$.

Instead of the linear function $\sum_{u \in \mathrm{pa}(v)} w_{v, u} X_u$, we use the following nonlinear function:
\[
\xi\!\left(\sum_{u \in \mathrm{pa}(v)} w_{v, u} X_u\right),
\]
where nonlinear function $\xi(t)$ is randomly chosen from $\xi(t) = \mathrm{sin}(t)$, $\xi(t) = \mathrm{cos}(t)$, and $\xi(t) = \mathrm{tanh}(t)$.

In our additional experiments for the cluster-size sensitivity analysis (\Cref{subsec-add-sensitivity-cluster-size}), we also consider unbounded nonlinear functions. Specifically, we choose $\xi(t)$ from 
\[ \xi(t)=t+0.05t|t|,\qquad \xi(t)=t+0.01t^3,\qquad \xi(t)=t\sigma(t), \] 
where $\sigma(t)=1/(1+\exp(-t))$ is the sigmoid function. These functions correspond to a signed quadratic function, a mild cubic function, and the sigmoid linear unit (SiLU), respectively.

	\begin{figure}[t]
		\includegraphics[height=5.8cm]{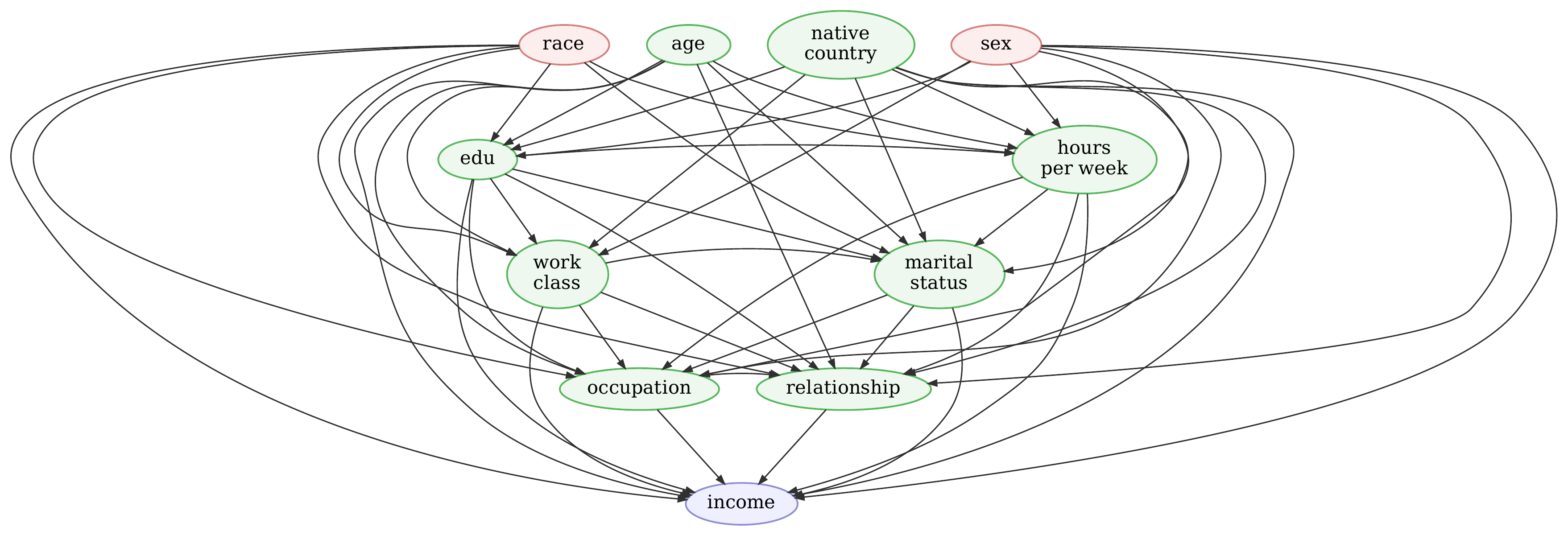}
		\centering 
			\caption{The variable-level causal graph for the Adult dataset \citep{zhang2017achieving}}
		\label{fig:adult-graph}
	\end{figure}

  	\begin{figure}[t]
		\includegraphics[height=4cm]{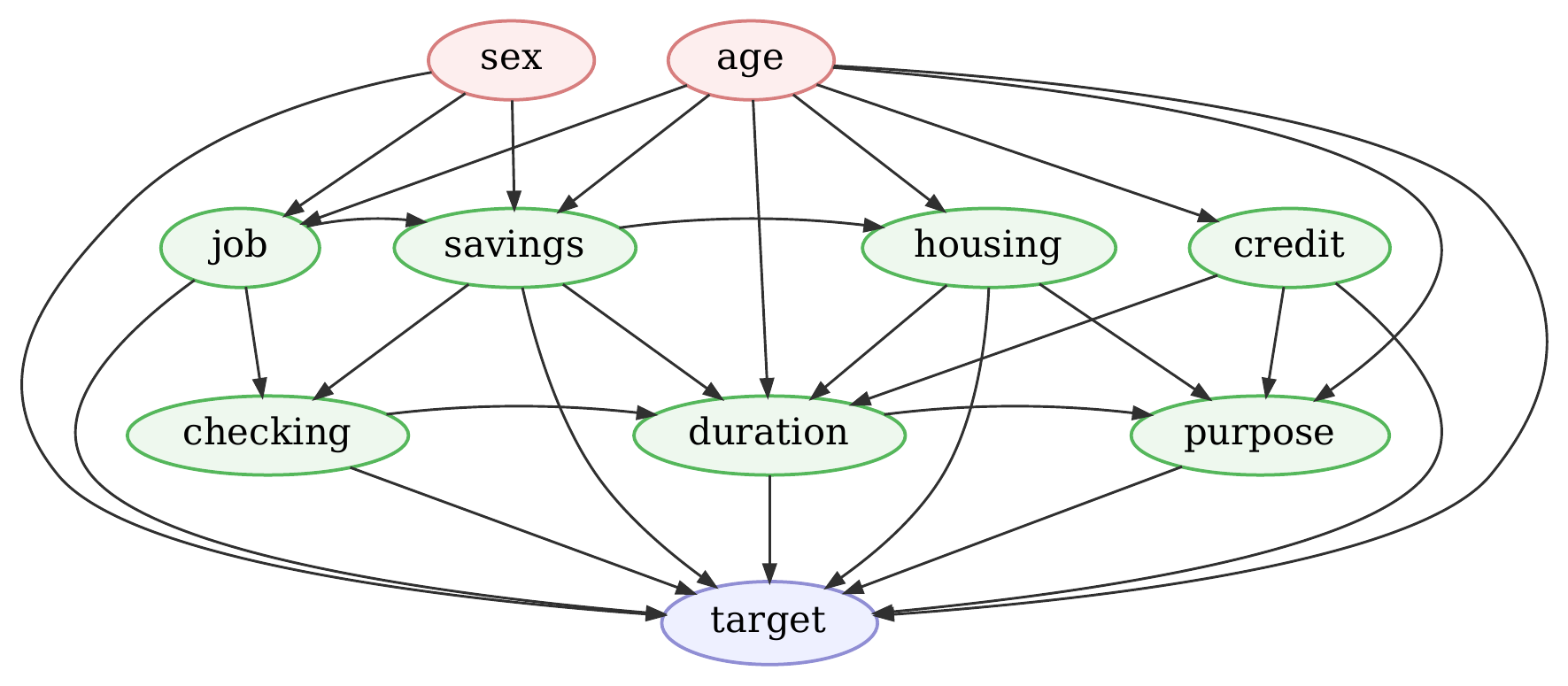}
		\centering 
			\caption{The variable-level causal graph for the German dataset \citep{watson2021local}}
		\label{fig:german-graph}
	\end{figure}

    	\begin{figure}[t]
		\includegraphics[height=4cm]{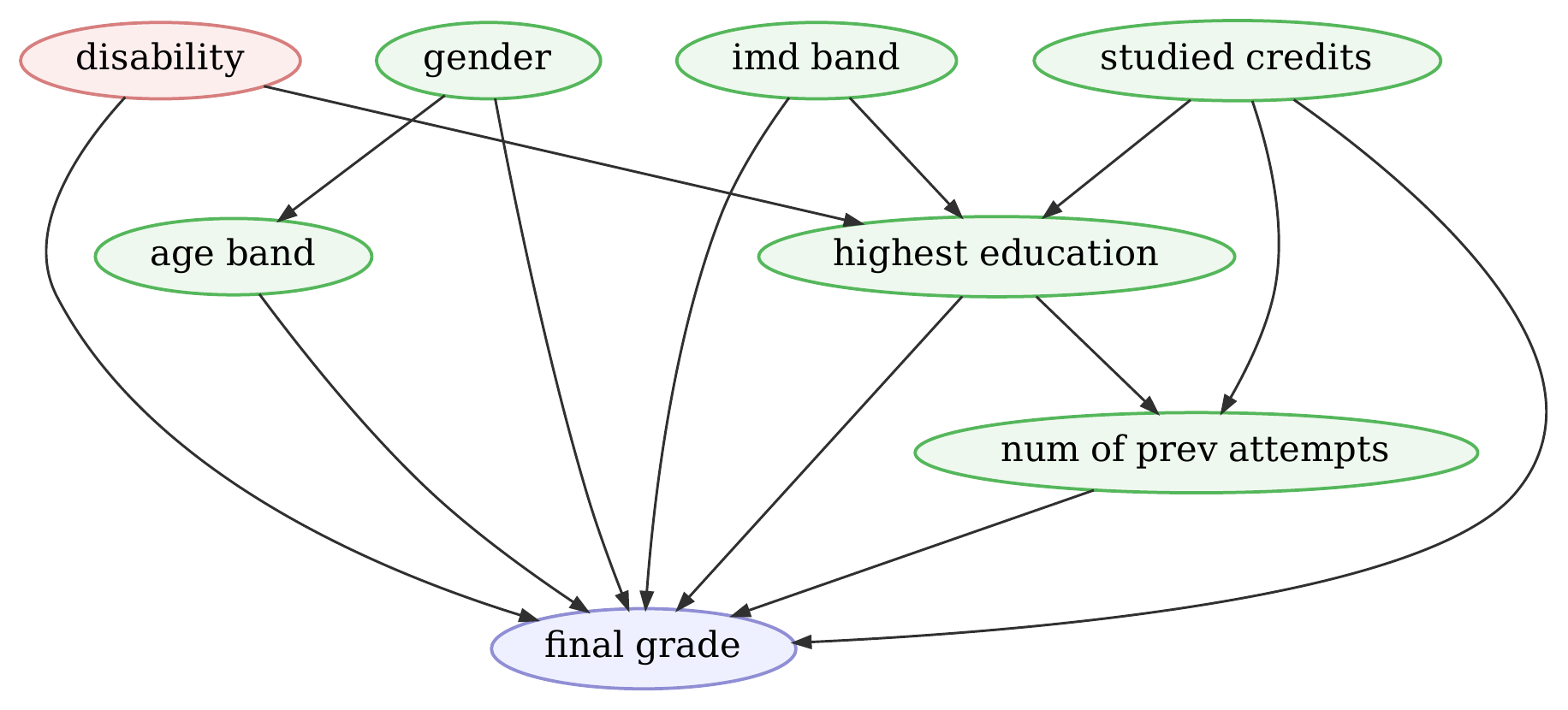}
		\centering 
			\caption{The variable-level causal graph for the OULAD dataset \citep{kim2025counterfactual}}
		\label{fig:oulad-graph}
	\end{figure}

\subsection{Real-World Data} \label{subsec-realdata}

We use three standard benchmarks: Adult, German credit, and the Open University Learning Analytics Dataset (OULAD), downloaded from the University of California, Irvine (UCI) repository \citep{uci}.

\paragraph{Adult Dataset.} The Adult dataset contains $48,842$ instances with $14$ attributes, where the task is to predict whether an individual's income exceeds \$50,000 per year. We treat gender $A_1$ and race $A_2$ as sensitive features $\biA = [A_1, A_2]$,
where $A_1 = 0$ and $A_1 = 1$ represent female and male, respectively, and $A_2 = 0$ and $A_2 = 1$ represent not White (i.e., all other race categories grouped) and White, respectively.
We manually construct a cluster partition based on the domain knowledge of the Adult dataset. Specifically, we group related features into clusters by grouping the attributes of family background (i.e., \texttt{marital-status} and \texttt{relationship}) as one cluster, and the attributes of demographic information (i.e., \texttt{education} and \texttt{native-country}) as another cluster. Other features are treated as singleton-node clusters.

\paragraph{German Credit Dataset.} The German credit dataset contains $1,000$ instances with $20$ attributes, where the task is to classify people described by a set of attributes as good or bad credit risks. We treat gender and age as sensitive features $\biA$.
We group the attributes of financial information (i.e., \texttt{credit amount} and \texttt{duration}) as one cluster. Other features are treated as singleton-node clusters.

\paragraph{OULAD Dataset.} The OULAD dataset contains the records about $32,593$ students with $7$ attributes, where the task is to predict whether a student will pass a course. We treat disability as the sensitive feature $\biA$.
Here we treat all features as singleton-node clusters, as the OULAD dataset contains only a small number of features.

\paragraph{Causal Graphs.} 
Since the ground-truth causal graph structures are \textbf{unknown} for these real-world datasets, we use the causal graphs provided by previous works that have been constructed based on domain knowledge or data-driven causal discovery methods.
For the Adult, German credit, and OULAD datasets,
we use the causal graphs illustrated in \Cref{fig:adult-graph,fig:german-graph,fig:oulad-graph}, which are provided by \citet{zhang2017achieving}, \citet{watson2021local}, and \citet{kim2025counterfactual}, respectively. 
Based on these causal graphs, we identify the back-door adjustment sets to use the IPW-based estimator for the kernel MMD evaluation.

\subsection{Performance Metrics} \label{subsec-metrics}

To evaluate the performance of each method, we use the following metrics:
\begin{itemize}[leftmargin=0.5cm]
  \item Root mean squared error (\textbf{RMSE}), which measures the error between predicted and true target values.
  \item Area under the (receiver operating characteristic) curve (\textbf{AUC}), which measures the classification performance. By focusing on the ranking of predictions rather than absolute accuracy, it remains robust even with imbalanced datasets.
  \item \textbf{Unfairness}, which measures the interventional unfairness based on the kernel MMD between interventional distributions of prediction $\biYhat$ under different sensitive feature values. In our synthetic data experiments, we estimate the empirical kernel MMDs by sampling the values of input features $\biX$ under interventions from the true interventional distributions obtained with the ground-truth SCM parameters. In real-world data experiments, we estimate the kernel MMDs using IPW (\Cref{subsubsec:penalty_formulation}) with valid adjustment set $\biZ$ identified from the provided causal graph structures (\Cref{subsec-realdata}). 
\end{itemize}

\subsection{Computing Infrastructure} \label{sec-computing}
	
In our runtime computation experiments (\Cref{subsec-add-graph-inference}), we used a 64-bit Ubuntu machine with 2.30GHz Xeon Gold 5218 32-core (x2) CPUs, NVIDIA TITAN RTX (x4) GPUs, and 256-GB RAM.

\section{Additional Experiments} \label{sec-add-experiments}

This section presents additional experimental results for further validating the effectiveness of our \textbf{C-IFair} method.

\subsection{Experiments with Admissible Features} \label{subsec-add-admissible}

\begin{table*}[t]
  \centering
  \caption{RMSE and unfairness for linear (top) and nonlinear (bottom) datasets with admissible features on held-out test set.}
  \label{table:synth_rmse_unfairness_admissible}
  \scshape
  \scalebox{0.92}{%
  \begin{tabular}{lcccccc}
    \toprule
    & \multicolumn{2}{c}{$d$=5 ($d_v$=15)} & \multicolumn{2}{c}{$d$=10 ($d_v$=30)} & \multicolumn{2}{c}{$d$=15 ($d_v$=45)} \\
    Linear & \textsc{RMSE}$\downarrow$ & \textsc{Unfairness}$\downarrow$
    & \textsc{RMSE}$\downarrow$ & \textsc{Unfairness}$\downarrow$
    & \textsc{RMSE}$\downarrow$ & \textsc{Unfairness}$\downarrow$ \\
    \midrule
    \textsc{Oracle} 
      & $0.985 \pm 0.056$ & $0.000 \pm 0.000$
      & $0.795 \pm 0.090$ & $0.000 \pm 0.000$
      & $0.879 \pm 0.093$ & $0.000 \pm 0.000$ \\
    \midrule
    \textsc{Full} 
      & $0.559 \pm 0.090$ & $0.244 \pm 0.112$
      & $0.553 \pm 0.148$ & $0.065 \pm 0.022$
      & $0.657 \pm 0.430$ & $0.071 \pm 0.054$ \\
    \midrule
    \textsc{Unaware} 
      & $0.805 \pm 0.104$ & $0.083 \pm 0.062$
      & $0.773 \pm 0.211$ & $0.035 \pm 0.020$
      & $0.917 \pm 0.141$ & $0.034 \pm 0.023$ \\
    \textsc{No-DesCs} 
      & $0.797 \pm 0.112$ & $0.085 \pm 0.065$
      & $0.591 \pm 0.105$ & $0.038 \pm 0.008$
      & $0.759 \pm 0.191$ & $0.064 \pm 0.056$ \\
    $\epsilon$-\textsc{IFair} 
      & $0.786 \pm 0.043$ & $0.099 \pm 0.082$
      & $0.600 \pm 0.064$ & $0.026 \pm 0.011$
      & $0.750 \pm 0.103$ & $0.039 \pm 0.003$ \\
    \textbf{\textsc{C-IFair}} 
      & $0.777 \pm 0.061$ & $0.058 \pm 0.060$
      & $0.581 \pm 0.087$ & $0.015 \pm 0.009$
      & $0.728 \pm 0.076$ & $0.024 \pm 0.026$ \\
    \bottomrule
  \end{tabular}%
  } 

  \scalebox{0.92}{%
  \begin{tabular}{lcccccc}
    \toprule
    & \multicolumn{2}{c}{$d$=5 ($d_v$=15)} & \multicolumn{2}{c}{$d$=10 ($d_v$=30)} & \multicolumn{2}{c}{$d$=15 ($d_v$=45)} \\
    Nonlinear & \textsc{RMSE}$\downarrow$ & \textsc{Unfairness}$\downarrow$
    & \textsc{RMSE}$\downarrow$ & \textsc{Unfairness}$\downarrow$
    & \textsc{RMSE}$\downarrow$ & \textsc{Unfairness}$\downarrow$ \\
    \midrule
    \textsc{Oracle} 
      & $1.099 \pm 0.056$ & $0.000 \pm 0.000$
      & $1.083 \pm 0.086$ & $0.000 \pm 0.000$
      & $1.071 \pm 0.105$ & $0.000 \pm 0.000$ \\
    \midrule
    \textsc{Full} 
      & $0.953 \pm 0.144$ & $0.089 \pm 0.062$
      & $0.973 \pm 0.124$ & $0.062 \pm 0.010$
      & $0.989 \pm 0.226$ & $0.055 \pm 0.006$ \\
    \midrule
    \textsc{Unaware} 
      & $0.994 \pm 0.064$ & $0.074 \pm 0.057$
      & $0.959 \pm 0.086$ & $0.032 \pm 0.010$
      & $1.065 \pm 0.129$ & $0.025 \pm 0.010$ \\
    \textsc{No-DesCs} 
      & $0.991 \pm 0.063$ & $0.061 \pm 0.047$
      & $0.948 \pm 0.085$ & $0.023 \pm 0.011$
      & $1.071 \pm 0.170$ & $0.027 \pm 0.017$ \\
    $\epsilon$-\textsc{IFair} 
      & $0.996 \pm 0.070$ & $0.097 \pm 0.019$
      & $0.960 \pm 0.078$ & $0.025 \pm 0.004$
      & $1.019 \pm 0.046$ & $0.021 \pm 0.009$ \\
    \textbf{\textsc{C-IFair}} 
      & $0.989 \pm 0.062$ & $0.046 \pm 0.081$
      & $0.947 \pm 0.104$ & $0.021 \pm 0.002$
      & $1.003 \pm 0.075$ & $0.019 \pm 0.007$ \\
    \bottomrule
  \end{tabular}%
  } 
\end{table*}

Admissible features are the features through which the effect of sensitive features $\biA$ is \textbf{not} treated as unfair, due to the requirement of the application domain (see the motivating example scenario illustrated in \Cref{subsec:causal_notions}).
Although the use of admissible features can be beneficial for improving the prediction performance, it becomes more challenging to achieve interventional fairness, as the presence of admissible features increases the number of interventional distribution pairs.

Below we evaluate the performance of \textbf{C-IFair} method in the presence of such admissible features. 

\paragraph{Data.} We use linear and nonlinear synthetic datasets generated in the same way as described in \Cref{subsec-synthdata}, except that we randomly choose admissible feature cluster $\biX^{\mathrm{ad}}$ of 3 binary variables from the children of the sensitive cluster $\biA$ in the cluster DAG $\cDAG$.

\paragraph{Baselines.} In this experiment, we do not compare \textbf{C-IFair} method with $\ell$\textbf{-IFair}, as its publicly available implementation does not support the setting with admissible features. For the remaining five baselines presented in \Cref{subsec-baselines}, we use the same settings as described in \Cref{subsec-setup}.

\paragraph{Results.} \Cref{table:synth_rmse_unfairness_admissible} presents the mean and standard deviation of RMSE and unfairness over 20 synthetic datasets with admissible features. 

Despite the increase in the number of interventional distribution pairs, our \textbf{C-IFair} again achieves the best performance in terms of both RMSE and unfairness across all settings, underscoring the effectiveness of our unfairness penalty formulation. 
Notably, compared with $\epsilon$\textbf{-IFair},
which aims to reduce the kernel MMD estimated by approximately sampling interventional datasets for each value pair of intervened variables $\biA$ and $\biX^{\mathrm{ad}}$, 
our \textbf{C-IFair} method consistently yields fairer predictions (particularly on linear datasets), 
highlighting the advantage of our barycenter kernel MMD formulation in handling the increased number of interventional distribution pairs.

\subsection{Experiments under Dense Cluster Graphs} \label{subsec-add-dense}

\begin{table*}[t]
  \centering
  \caption{RMSE and unfairness on synthetic datasets generated using dense causal graphs}
  \label{table:dense_synth_rmse_unfairness}
  \scshape

  \scalebox{0.92}{%
  \begin{tabular}{lcccc}
    \toprule
    & \multicolumn{2}{c}{$d$=10} & \multicolumn{2}{c}{$d$=15} \\
    \textsc{Linear}
    & \textsc{RMSE}$\downarrow$ & \textsc{Unfairness}$\downarrow$
    & \textsc{RMSE}$\downarrow$ & \textsc{Unfairness}$\downarrow$ \\
    \midrule
    \textsc{Oracle} 
      & $0.707 \pm 0.147$ & $0.000 \pm 0.000$
      & $0.727 \pm 0.173$ & $0.000 \pm 0.000$ \\
    \midrule
    \textsc{Full} 
      & $0.470 \pm 0.127$ & $0.099 \pm 0.079$
      & $0.492 \pm 0.113$ & $0.061 \pm 0.045$ \\
    \midrule
    \textsc{Unaware} 
      & $0.730 \pm 0.222$ & $0.048 \pm 0.029$
      & $0.928 \pm 0.066$ & $0.023 \pm 0.049$ \\
    \textsc{No-DesCs} 
      & $0.571 \pm 0.114$ & $0.047 \pm 0.055$
      & $0.504 \pm 0.096$ & $0.026 \pm 0.013$ \\
    $\epsilon$-\textsc{IFair} 
      & $0.689 \pm 0.099$ & $0.124 \pm 0.060$
      & $0.503 \pm 0.073$ & $0.302 \pm 0.109$ \\
    $\ell$-\textsc{IFair} 
      & $0.669 \pm 0.295$ & $0.077 \pm 0.053$
      & $0.767 \pm 0.241$ & $0.044 \pm 0.033$ \\
    \textbf{\textsc{C-IFair}} 
      & $0.550 \pm 0.074$ & $0.040 \pm 0.061$
      & $0.498 \pm 0.121$ & $0.020 \pm 0.011$ \\
    \bottomrule
  \end{tabular}%
  }

  \scalebox{0.92}{%
  \begin{tabular}{lcccc}
    \toprule
    & \multicolumn{2}{c}{$d$=10} & \multicolumn{2}{c}{$d$=15} \\
    \textsc{Nonlinear}
    & \textsc{RMSE}$\downarrow$ & \textsc{Unfairness}$\downarrow$
    & \textsc{RMSE}$\downarrow$ & \textsc{Unfairness}$\downarrow$ \\
    \midrule
    \textsc{Oracle} 
      & $0.976 \pm 0.045$ & $0.000 \pm 0.000$
      & $0.978 \pm 0.041$ & $0.000 \pm 0.000$ \\
    \midrule
    \textsc{Full} 
      & $0.807 \pm 0.051$ & $0.022 \pm 0.010$
      & $0.870 \pm 0.054$ & $0.019 \pm 0.006$ \\
    \midrule
    \textsc{Unaware} 
      & $0.988 \pm 0.059$ & $0.021 \pm 0.024$
      & $0.996 \pm 0.031$ & $0.005 \pm 0.005$ \\
    \textsc{No-DesCs} 
      & $0.990 \pm 0.052$ & $0.016 \pm 0.015$
      & $0.985 \pm 0.057$ & $0.006 \pm 0.002$ \\
    $\epsilon$-\textsc{IFair} 
      & $0.999 \pm 0.019$ & $0.024 \pm 0.005$
      & $0.916 \pm 0.051$ & $0.077 \pm 0.046$ \\
    $\ell$-\textsc{IFair} 
      & $0.989 \pm 0.056$ & $0.014 \pm 0.001$
      & $1.082 \pm 0.002$ & $0.006 \pm 0.000$ \\
    \textbf{\textsc{C-IFair}} 
      & $0.985 \pm 0.027$ & $0.011 \pm 0.002$
      & $0.976 \pm 0.023$ & $0.005 \pm 0.002$ \\
    \bottomrule
  \end{tabular}%
  }
\end{table*}
	\begin{figure}[t]
		\includegraphics[height=4cm]{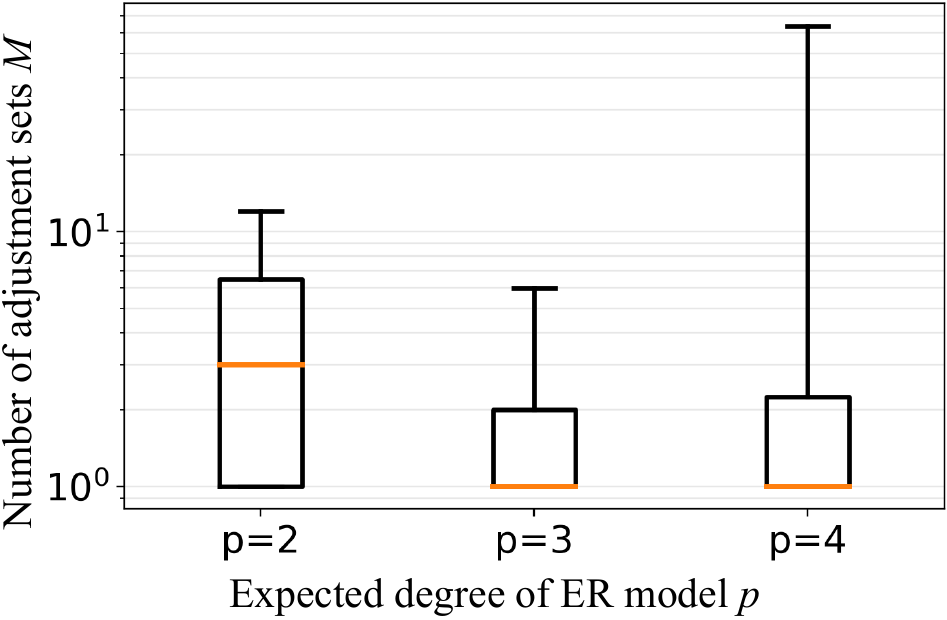}
		\centering 
			\caption{Log-scale box plot of number of adjustment sets $\npp$ for expected node degree $p$ on linear datasets with $d=10$ clusters. Whiskers, boxes, and orange lines denote minimum and maximum,  interquartile range, and median, respectively.} 
		\label{fig:box-plot}
	\end{figure}

To measure unfairness, our \textbf{C-IFair} method infers interventional distributions using adjustment cluster sets identified from the cluster CPDAG (\Cref{subsec:adjustment_sets}). Since these sets rely on the number of possible parent sets for sensitive features $\biA$ (i.e., $S^1, \dots, S^{\npp}$), their number $\npp$ increases when the number of undirected edges incident to $\biA$ grows, which can make it more difficult to achieve interventional fairness. 

To test our \textbf{C-IFair} under such cases, we conduct experiments using synthetic datasets generated from dense graphs.

\paragraph{Data.} We use synthetic datasets generated in the same way as described in \Cref{subsec-synthdata}, except that we generate the underlying variable-level DAGs $\vDAG$ using the ER model with expected node degree $p = 4$ to make the cluster graphs denser.
We set the number of clusters to $d=10, 15$, as setting $d=5$ substantially reduces the randomness of the ER model.  

\paragraph{Graph Density and Number of Adjustment Sets.}
Before making performance comparisons, we analyze the relationship between the density of cluster graphs and the number of adjustment sets $\npp$.

\Cref{fig:box-plot} presents the log-scale box plot of the number of adjustment sets $\npp$ for expected node degree $p = 2, 3$, and $4$ in the ER model on the linear synthetic datasets with $d = 10$ clusters.

Surprisingly, on average, the number of adjustment sets $\npp$ decreases as the expected node degree $p$ increases. 
This is mainly because the number of adjustment sets $\npp$ is determined 
not necessarily by the overall graph density, but rather 
by the number of \textbf{undirected} edges incident to sensitive features $\biA$. 

In particular, when we perform graph refinement (\Cref{subsubsec:additional_clusters}), the number of undirected edges in the cluster CPDAG can sharply decrease,
implying that cluster refinement can also help disambiguate the causal relationships around sensitive features $\biA$ and thus reduce the number of adjustment sets $\npp$.
This observation is consistent with 
the example of variable-level and cluster CPDAGs in \Cref{fig:toy_var_graph}~(c) and (e), where the cluster CPDAG has 2 undirected edges incident to $\biA$, whereas the variable-level CPDAG has only 1 undirected edge incident to $\biA$, despite the increase in the graph size.

However, on some datasets with $p=4$, graph refinement does not occur, and the number of adjustment sets $\npp$ can be as large as $\npp = 64$. Below we compare the performance of each method under $p=4$.

\paragraph{Results.} \Cref{table:dense_synth_rmse_unfairness} presents the mean and standard deviation of RMSE and unfairness over 20 synthetic datasets.

Our \textbf{C-IFair} method again achieves the best performance in terms of both RMSE and unfairness across all settings, 
demonstrating its effectiveness even under dense cluster graphs.
Although $\ell$\textbf{-IFair} also performs parent-set-based adjustment using a variable-level CPDAG,
it produces unfairer predictions than our \textbf{C-IFair} method, particularly on linear datasets, 
which may be attributed to the fact that the variable-level CPDAGs inferred by the PC algorithm are less accurate than the cluster CPDAGs inferred by the CLOC algorithm, as we will describe in \Cref{subsec-add-graph-inference}.

Notably, $\epsilon$-\textbf{IFair} yields the worst performance in terms of unfairness in all settings, 
due to the significant violation of its assumption for adjustment, 
which requires that no features in $\biX$ are connected to sensitive features $\biA$ via undirected edges in the variable-level CPDAG.
By contrast, our \textbf{C-IFair} method does not require such a demanding assumption,
and thus can strike the best balance between prediction accuracy and fairness, even under dense cluster graphs.

\begin{table*}[t]
  \centering
  \caption{RMSE and unfairness on linear synthetic datasets: For our \textbf{C-IFair} (and \textbf{No-DesCs}), we use a cluster CPDAG inferred with an inadmissible partition.}
  \label{table:linear_inadmissible_partition}
  \scshape

  \scalebox{0.92}{%
  \begin{tabular}{lcccccc}
    \toprule
    & \multicolumn{2}{c}{$d$=5} & \multicolumn{2}{c}{$d$=10} & \multicolumn{2}{c}{$d$=15} \\
    \textsc{Method}
    & \textsc{RMSE}$\downarrow$ & \textsc{Unfairness}$\downarrow$
    & \textsc{RMSE}$\downarrow$ & \textsc{Unfairness}$\downarrow$
    & \textsc{RMSE}$\downarrow$ & \textsc{Unfairness}$\downarrow$ \\
    \midrule
    \textsc{Oracle} 
      & $0.949 \pm 0.096$ & $0.000 \pm 0.000$
      & $0.758 \pm 0.115$ & $0.000 \pm 0.000$
      & $0.702 \pm 0.123$ & $0.000 \pm 0.000$ \\
    \midrule
    \textsc{Full} 
      & $0.484 \pm 0.203$ & $0.239 \pm 0.218$
      & $0.522 \pm 0.174$ & $0.100 \pm 0.073$
      & $0.617 \pm 0.207$ & $0.073 \pm 0.062$ \\
    \midrule
    \textsc{Unaware} 
      & $0.709 \pm 0.191$ & $0.071 \pm 0.061$
      & $0.753 \pm 0.156$ & $0.034 \pm 0.051$
      & $0.782 \pm 0.121$ & $0.035 \pm 0.070$ \\
    \textsc{No-DesCs} 
      & $0.694 \pm 0.188$ & $0.069 \pm 0.050$
      & $0.615 \pm 0.135$ & $0.042 \pm 0.039$
      & $0.631 \pm 0.142$ & $0.031 \pm 0.028$ \\
    $\epsilon$-\textsc{IFair} 
      & $0.718 \pm 0.124$ & $0.172 \pm 0.004$
      & $0.654 \pm 0.152$ & $0.097 \pm 0.015$
      & $0.561 \pm 0.093$ & $0.091 \pm 0.016$ \\
    $\ell$-\textsc{IFair} 
      & $0.719 \pm 0.332$ & $0.064 \pm 0.006$
      & $0.703 \pm 0.004$ & $0.127 \pm 0.104$
      & $0.624 \pm 0.030$ & $0.086 \pm 0.027$ \\
    \textsc{C-IFair} 
      & $0.705 \pm 0.189$ & $0.072 \pm 0.072$
      & $0.622 \pm 0.117$ & $0.041 \pm 0.038$
      & $0.637 \pm 0.177$ & $0.027 \pm 0.028$ \\
    \bottomrule
  \end{tabular}%
  }
\end{table*}

\subsection{Empirical Robustness to Assumption Violation} \label{subsec-add-violate}

Our \textbf{C-IFair} method relies on the cluster CPDAG inferred by the CLOC algorithm \citep{anand2025causal}, which requires the three assumptions described in \Cref{subsubsec:overview}. 
Compared with the assumptions required by $\epsilon$\textbf{-IFair} and $\ell$\textbf{-IFair}, Assumption~\ref{asmp:cpdag} is an additional assumption that is not needed by these methods.

This assumption requires that the cluster partition used for inferring the cluster CPDAG is admissible, 
which means that partitioning the nodes in the variable-level causal DAG does not yield any directed cycle in the cluster graph.
Since the partition is given by a user in practical applications,
this assumption can be violated when the domain knowledge of the relationships among features is limited.
Theoretically, CLOC cannot guarantee the identifiability of the cluster CPDAG in such cases, and thus the validity of the adjustment sets identified from the cluster CPDAG is not guaranteed.

Below we empirically evaluate the robustness of our \textbf{C-IFair} to the violation of Assumption~\ref{asmp:cpdag}, using linear synthetic datasets.

\textbf{Data and Causal Graphs.} We generate linear synthetic datasets in the same way as described in \Cref{subsec-synthdata}.
Based on the underlying variable-level DAGs $\vDAG$ generated from the ER model with expected node degree $2$,
we force the cluster partition to be \textit{inadmissible} by taking two steps:
\begin{itemize}
  \item We seek two node pairs $(U_i, V_i)$ and $(U_j, V_j)$ from the variable-level DAG $\vDAG$ such that $U_i \to U_j$ and $V_i \leftarrow V_j$.
  \item Then we force the cluster partition to assign $(U_i, V_i)$ to one cluster $\biC_i$ and $(U_j, V_j)$ to another cluster $\biC_j$, which yields a directed cycle (i.e., $\biC_i \to \biC_j$ and $\biC_i \leftarrow \biC_j$) in the ground-truth cluster graph.
\end{itemize}
Using such partition $\cNode = \{\biC_1, \dots, \biC_d\}$, we infer a cluster CPDAG using the CLOC algorithm, and then apply our \textbf{C-IFair} method (and \textbf{No-DesCs}) using the inferred cluster CPDAG.

To focus on the violation of Assumption~\ref{asmp:cpdag}, we do \textbf{not} use nonlinear synthetic datasets, 
as the complex nonlinearity affects the accuracy of conditional independence tests used for inferring the cluster CPDAG, 
which can make it difficult to isolate the effect of the violation of Assumption~\ref{asmp:cpdag} on the performance of our \textbf{C-IFair} method.

\textbf{Results.} \Cref{table:linear_inadmissible_partition} presents the mean and standard deviation of RMSE and unfairness over 20 linear synthetic datasets.

As expected, the performance of our \textbf{C-IFair} method decreases compared with the results in \Cref{table:synth_rmse_unfairness} (top), due to the inaccurate cluster CPDAG inferred with the inadmissible partition.
However, our \textbf{C-IFair} method still achieves highly comparable performance to the baselines, 
and even outperforms some of them in terms of RMSE and unfairness,
demonstrating its empirical robustness to the violation of Assumption~\ref{asmp:cpdag}.
This might be attributed to the fact that, even with the violation of Assumption~\ref{asmp:cpdag}, the cluster CPDAG inferred by CLOC can still capture some of the true causal relationships among clusters, which can help identify some valid adjustment sets for estimating interventional distributions.

\begin{table*}[t]
  \centering
  \caption{RMSE and unfairness on held-out test set in linear synthetic datasets generated using causal DAG in \Cref{fig:conn_graph}~(a).}
  \label{table:unadjustable_rmse_unfairness}
  \scshape
  \scalebox{1.0}{%
  \begin{tabular}{lcc}
    \toprule
    & \textsc{RMSE}$\downarrow$ & \textsc{Unfairness}$\downarrow$ \\
    \midrule
    \textsc{Oracle} 
      & $0.756 \pm 0.053$ & $0.000 \pm 0.000$ \\
    \midrule
    \textsc{Full} 
      & $0.256 \pm 0.051$ & $0.093 \pm 0.083$ \\
    \midrule
    \textsc{Unaware} 
      & $0.707 \pm 0.158$ & $0.054 \pm 0.001$ \\
    \textsc{No-DesCs} 
      & $0.708 \pm 0.158$ & $0.044 \pm 0.001$ \\
    $\epsilon$-\textsc{IFair} 
      & $0.435 \pm 0.132$ & $0.026 \pm 0.004$ \\
    $\ell$-\textsc{IFair} 
      & $0.753 \pm 0.224$ & $0.027 \pm 0.019$ \\
    \textbf{\textsc{C-IFair}} 
      & $0.320 \pm 0.115$ & $0.025 \pm 0.013$ \\
    \bottomrule
  \end{tabular}%
  }
\end{table*}

	\begin{figure*}[t]
		\includegraphics[height=5.2cm]{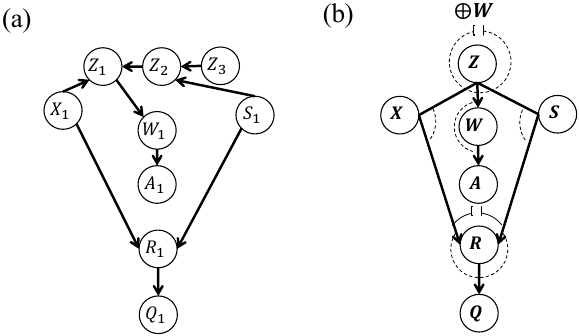}
		\centering 
			\caption{Specially designed causal graphs where parent of sensitive features $\biA$ is annotated in connection mark: (a) variable-level ground-truth DAG over $\biX$, (b) cluster CPDAG over $\biX$ with independence arcs and connection marks.} 
		\label{fig:conn_graph}
	\end{figure*}

\subsection{Experiments involving Connection Marks} \label{subsec-add-unadjustable}

As described in \Cref{subsec:adjustment_sets}, our strategy for obtaining adjustment sets is two-fold. We first enumerate a parent set of sensitive features $\biA$. If this set contains the clusters annotated with connection marks, then we augment it with additional clusters. 

To validate this second stage, we evaluate the performance on synthetic datasets generated using the causal DAG in \Cref{fig:conn_graph}~(a), where the parent of $\biA$ (i.e., $\biW$) appears in a connection mark in the corresponding cluster CPDAG (\Cref{fig:conn_graph}~(b)). 

\paragraph{Data and Causal Graphs.} We generate linear synthetic datasets using the same data generation process described in \Cref{subsec-synthdata}, except that we use the causal DAG in \Cref{fig:conn_graph}~(a) as the ground-truth variable-level DAG $\vDAG$. 

We choose this causal graph from \citet[Figure 4~(c)]{anand2025causal}, as its corresponding cluster CPDAG contains a cluster with a connection mark. We illustrate the cluster CPDAG in \Cref{fig:conn_graph}~(b), as well as the independence arcs and connection marks. As shown in \Cref{fig:conn_graph}, the number of variables and clusters are $d_v=9$ and $d=7$, respectively.

\paragraph{Results.} \Cref{table:unadjustable_rmse_unfairness} presents the mean and standard deviation of RMSE and unfairness over 20 synthetic datasets.

Our \textbf{C-IFair} strikes the best balance between RMSE and unfairness across baselines, 
implying that our adjustment cluster addition algorithm (\Cref{alg:add-adj-cluster-set}) successfully works to identify the valid adjustment sets.

\subsection{Graph Inference Performance Comparison} \label{subsec-add-graph-inference}

To illustrate why our \textbf{C-IFair} yields better performance than the baseline based on the variable-level CPDAG (i.e., $\ell$\textbf{-IFair}), we compare the performance of the CLOC algorithm with the PC algorithm in terms of both accuracy and runtime.

\paragraph{Baseline.} We compare the CLOC algorithm with the \textit{PC-then-cluster} approach in \citet{anand2025causal}, which first applies the PC algorithm to learn a variable-level CPDAG and then constructs a cluster CPDAG by grouping variables into clusters based on the pre-defined cluster partition.

\paragraph{Fairness and Objective of Performance Comparison.} It is important to note that \textbf{we cannot make an entirely fair comparison} between the CLOC and PC algorithms, as their inference targets are \textbf{different}. 
As noted by \citet{anand2025causal}, the PC-then-cluster approach is \textbf{not} guaranteed to ensure that the learned cluster CPDAG belongs to the cluster MEC, even when all conditional independence tests are correct. 

Moreover, we cannot easily jump to the conclusion that the CLOC's better performance implies the better performance of our \textbf{C-IFair}, as the performance of $\ell$\textbf{-IFair} does \textbf{not} depend on the performance of PC-then-cluster approach but on the performance of the PC algorithm itself.

Our objective here is to give an illustration of the potential performance gap between the CLOC and naive PC-then-cluster algorithms, which could be treated as indirect evidence for the better performance of our \textbf{C-IFair} than $\ell$\textbf{-IFair}.

\paragraph{Results.} \Cref{fig:comparison_with_PC} presents the comparison of SHD and runtime between the CLOC and PC-then-cluster algorithms on $20$ linear synthetic datasets.

As expected, the CLOC algorithm outperforms the PC-then-cluster approach in high-dimensional setups ($d=15$), where our \textbf{C-IFair} also achieves better performance than $\ell$\textbf{-IFair}. Notably, the runtime of the PC-then-cluster approach is substantially higher than that of the CLOC algorithm and highly variable across datasets particularly in high-dimensional setups ($d=15$).

Such superior and stable performance justifies the use of the CLOC algorithm for cluster CPDAG inference in our \textbf{C-IFair}, which might be one of the key factors contributing to the better performance of our method than $\ell$\textbf{-IFair}.

 	\begin{figure*}[t]
		\includegraphics[height=4.5cm]{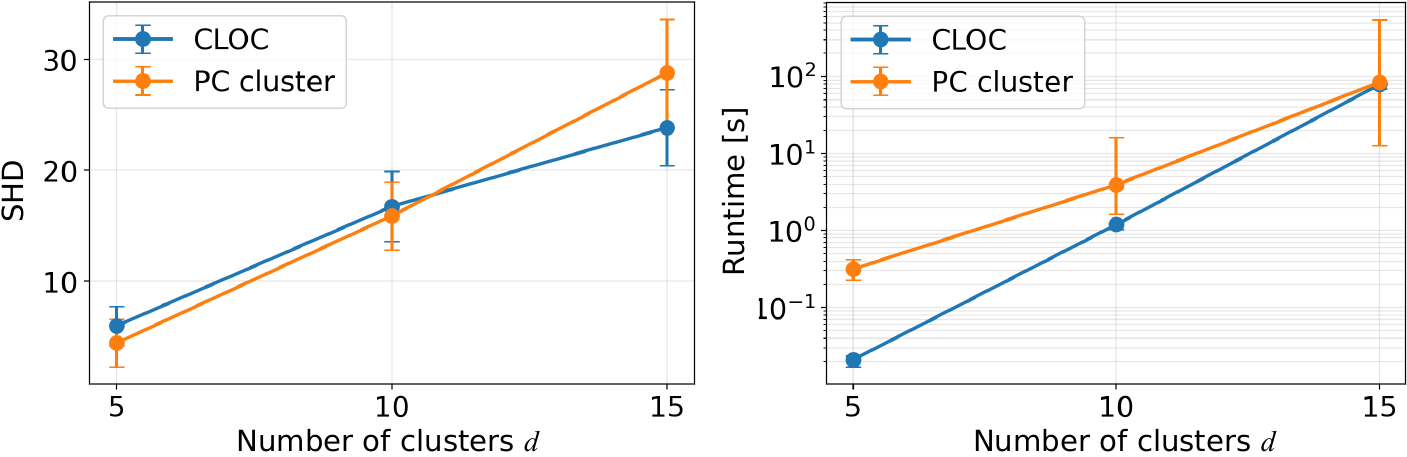}
		\centering 
			\caption{Performance comparison of CLOC with naive PC-based clustering on 20 linear synthetic datasets: (a) Mean and standard deviation of Structural Hamming distance (SHD) between the learned and true cluster DAGs; (b) Median and 25 / 75 percentiles of runtime [s] in a logarithmic scale.} 
		\label{fig:comparison_with_PC}
	\end{figure*}

  \paragraph{Discussion on time complexity.} 
  We conclude this section with a theoretical comparison of the worst-case time complexity of the CLOC and PC algorithms reported in the literature \citep{anand2025causal}.

  \citet{anand2025causal} show that the worst-case time complexity of the CLOC algorithm is $O(d^2 2^{p}(d + p^2))$, where $d$ is the number of clusters and $p$ is the maximum of cluster degrees. By contrast, the standard PC algorithm requires the worst-case time complexity of $O(d_v^2 2^{p_v})$ where $d_v$ is the number of variables, and $p_v$ is the maximum degree of the variable-level DAG.

  As discussed by \citet{anand2025causal},
  these complexity results suggest that 
  CLOC can be more efficient than PC when 
\begin{enumerate}
    \item the number of clusters $d$ is much smaller than the number of variables $d_v$, and
    \item inter-cluster connectivity is sparse (i.e., the maximum cluster degree $p$ is not larger).
\end{enumerate}
In real-world applications for algorithmic fairness, both of these conditions are often satisfied, as the number of clusters is typically much smaller than the number of variables, and inter-cluster connectivity is often sparse.

Hence, these theoretical insights imply that CLOC can be more efficient than the PC algorithm in practical applications. This efficiency justifies our refinement procedure for the cluster CPDAG using CLOC, described in \Cref{subsubsec:additional_clusters}.

\begin{table*}[t]
  \centering
  \caption{Performance comparison with and without oracle CPDAGs on synthetic datasets generated using the ER(2) graphs. An asterisk ($*$) indicates the use of an oracle variable-level or cluster-level CPDAG. Results are means $\pm$ standard deviations over 20 datasets.}
  \label{table:oracle_vs_without_oracle_ER2}
  \scshape

  \scalebox{0.97}{%
  \begin{tabular}{lcccccc}
    \toprule
    & \multicolumn{2}{c}{$d$=5 ($d_v$=15)} & \multicolumn{2}{c}{$d$=10 ($d_v$=30)} & \multicolumn{2}{c}{$d$=15 ($d_v$=45)} \\
    \textsc{Linear}
    & \textsc{RMSE}$\downarrow$ & \textsc{Unfairness}$\downarrow$
    & \textsc{RMSE}$\downarrow$ & \textsc{Unfairness}$\downarrow$
    & \textsc{RMSE}$\downarrow$ & \textsc{Unfairness}$\downarrow$ \\
    \midrule
    $\epsilon$-\textsc{IFair}$^*$
      & $0.655 \pm 0.048$ & $0.012 \pm 0.003$
      & $0.605 \pm 0.062$ & $0.014 \pm 0.009$
      & $0.597 \pm 0.052$ & $0.018 \pm 0.005$ \\
    $\epsilon$-\textsc{IFair}
      & $0.647 \pm 0.030$ & $0.069 \pm 0.031$
      & $0.663 \pm 0.062$ & $0.071 \pm 0.011$
      & $0.671 \pm 0.044$ & $0.040 \pm 0.011$ \\
    \midrule
    $\ell$-\textsc{IFair}$^*$
      & $0.869 \pm 0.039$ & $0.019 \pm 0.009$
      & $0.877 \pm 0.055$ & $0.018 \pm 0.006$
      & $0.712 \pm 0.027$ & $0.010 \pm 0.002$ \\
    $\ell$-\textsc{IFair}
      & $0.875 \pm 0.028$ & $0.069 \pm 0.081$
      & $0.964 \pm 0.049$ & $0.072 \pm 0.052$
      & $0.764 \pm 0.049$ & $0.052 \pm 0.052$ \\
    \midrule
    $C$-\textsc{IFair}$^*$
      & $0.596 \pm 0.055$ & $0.021 \pm 0.008$
      & $0.582 \pm 0.059$ & $0.014 \pm 0.007$
      & $0.629 \pm 0.075$ & $0.013 \pm 0.006$ \\
    $C$-\textsc{IFair}
      & $0.643 \pm 0.127$ & $0.060 \pm 0.054$
      & $0.660 \pm 0.123$ & $0.056 \pm 0.052$
      & $0.669 \pm 0.171$ & $0.020 \pm 0.036$ \\
    \bottomrule
  \end{tabular}%
  }

  \scalebox{0.97}{%
  \begin{tabular}{lcccccc}
    \toprule
    & \multicolumn{2}{c}{$d$=5 ($d_v$=15)} & \multicolumn{2}{c}{$d$=10 ($d_v$=30)} & \multicolumn{2}{c}{$d$=15 ($d_v$=45)} \\
    \textsc{Nonlinear}
    & \textsc{RMSE}$\downarrow$ & \textsc{Unfairness}$\downarrow$
    & \textsc{RMSE}$\downarrow$ & \textsc{Unfairness}$\downarrow$
    & \textsc{RMSE}$\downarrow$ & \textsc{Unfairness}$\downarrow$ \\
    \midrule
    $\epsilon$-\textsc{IFair}$^*$
      & $0.949 \pm 0.021$ & $0.032 \pm 0.021$
      & $1.019 \pm 0.017$ & $0.004 \pm 0.004$
      & $0.975 \pm 0.016$ & $0.001 \pm 0.001$ \\
    $\epsilon$-\textsc{IFair}
      & $0.938 \pm 0.006$ & $0.077 \pm 0.009$
      & $0.953 \pm 0.026$ & $0.027 \pm 0.000$
      & $0.962 \pm 0.016$ & $0.023 \pm 0.001$ \\
    \midrule
    $\ell$-\textsc{IFair}$^*$
      & $1.021 \pm 0.021$ & $0.004 \pm 0.001$
      & $1.193 \pm 0.046$ & $0.003 \pm 0.001$
      & $1.005 \pm 0.016$ & $0.002 \pm 0.001$ \\
    $\ell$-\textsc{IFair}
      & $0.987 \pm 0.065$ & $0.076 \pm 0.042$
      & $0.982 \pm 0.041$ & $0.023 \pm 0.010$
      & $1.042 \pm 0.041$ & $0.017 \pm 0.010$ \\
    \midrule
    $C$-\textsc{IFair}$^*$
      & $0.895 \pm 0.033$ & $0.045 \pm 0.047$
      & $0.875 \pm 0.061$ & $0.011 \pm 0.007$
      & $0.882 \pm 0.041$ & $0.004 \pm 0.002$ \\
    $C$-\textsc{IFair}
      & $0.929 \pm 0.051$ & $0.066 \pm 0.046$
      & $0.948 \pm 0.051$ & $0.021 \pm 0.035$
      & $0.960 \pm 0.036$ & $0.010 \pm 0.012$ \\
    \bottomrule
  \end{tabular}%
  }

\end{table*}

\begin{table*}[t]
  \centering
  \caption{Performance comparison with and without oracle CPDAGs on synthetic datasets generated using the ER(4) graphs. An asterisk ($*$) indicates the use of an oracle variable-level or cluster-level CPDAG. }
  \label{table:oracle_vs_without_oracle_ER4}
  \scshape

  \scalebox{1.0}{%
  \begin{tabular}{lcccc}
    \toprule
    & \multicolumn{2}{c}{$d$=10} & \multicolumn{2}{c}{$d$=15} \\
    \textsc{Linear}
    & \textsc{RMSE}$\downarrow$ & \textsc{Unfairness}$\downarrow$
    & \textsc{RMSE}$\downarrow$ & \textsc{Unfairness}$\downarrow$ \\
    \midrule
    $\epsilon$-\textsc{IFair}$^*$
      & $0.628 \pm 0.033$ & $0.015 \pm 0.013$
      & $0.548 \pm 0.107$ & $0.108 \pm 0.028$ \\
    $\epsilon$-\textsc{IFair}
      & $0.689 \pm 0.099$ & $0.124 \pm 0.060$
      & $0.503 \pm 0.073$ & $0.302 \pm 0.109$ \\
    \midrule
    $\ell$-\textsc{IFair}$^*$
      & $0.711 \pm 0.164$ & $0.032 \pm 0.011$
      & $0.685 \pm 0.117$ & $0.012 \pm 0.008$ \\
    $\ell$-\textsc{IFair}
      & $0.669 \pm 0.295$ & $0.077 \pm 0.053$
      & $0.767 \pm 0.241$ & $0.044 \pm 0.033$ \\
    \midrule
    $C$-\textsc{IFair}$^*$
      & $0.541 \pm 0.132$ & $0.006 \pm 0.001$
      & $0.511 \pm 0.087$ & $0.012 \pm 0.007$ \\
    $C$-\textsc{IFair}
      & $0.550 \pm 0.074$ & $0.040 \pm 0.061$
      & $0.498 \pm 0.121$ & $0.020 \pm 0.011$ \\
    \bottomrule
  \end{tabular}%
  }

  \scalebox{1.0}{%
  \begin{tabular}{lcccc}
    \toprule
    & \multicolumn{2}{c}{$d$=10} & \multicolumn{2}{c}{$d$=15} \\
    \textsc{Nonlinear}
    & \textsc{RMSE}$\downarrow$ & \textsc{Unfairness}$\downarrow$
    & \textsc{RMSE}$\downarrow$ & \textsc{Unfairness}$\downarrow$ \\
    \midrule
    $\epsilon$-\textsc{IFair}$^*$
      & $1.003 \pm 0.023$ & $0.002 \pm 0.001$
      & $0.976 \pm 0.001$ & $0.033 \pm 0.008$ \\
    $\epsilon$-\textsc{IFair}
      & $0.999 \pm 0.019$ & $0.024 \pm 0.005$
      & $0.916 \pm 0.051$ & $0.077 \pm 0.046$ \\
    \midrule
    $\ell$-\textsc{IFair}$^*$
      & $1.011 \pm 0.018$ & $0.001 \pm 0.001$
      & $1.046 \pm 0.041$ & $0.003 \pm 0.001$ \\
    $\ell$-\textsc{IFair}
      & $0.989 \pm 0.056$ & $0.014 \pm 0.001$
      & $1.082 \pm 0.002$ & $0.006 \pm 0.000$ \\
    \midrule
    $C$-\textsc{IFair}$^*$
      & $0.986 \pm 0.031$ & $0.003 \pm 0.001$
      & $0.972 \pm 0.014$ & $0.002 \pm 0.001$ \\
    $C$-\textsc{IFair}
      & $0.985 \pm 0.027$ & $0.011 \pm 0.002$
      & $0.976 \pm 0.023$ & $0.005 \pm 0.002$ \\
    \bottomrule
  \end{tabular}%
  }
\end{table*}

\subsection{Performance Comparison with/without Oracle CPDAGs} \label{subsec-add-oracle}

To investigate the performance difference due to the use of inferred variable-level or cluster CPDAGs, we test each of $\epsilon$-\textbf{IFair}, $\ell$-\textbf{IFair}, and \textbf{C-IFair} with and without oracle CPDAGs on synthetic datasets.

\paragraph{Data and Causal Graphs.} We use the linear and nonlinear synthetic datasets generated in \Cref{subsec-synthdata}. 
To examine robustness to graph structure, we use both the ER(2) and ER(4) models, i.e., ER models with expected node degrees $2$ and $4$, to generate the underlying variable-level DAGs.

\paragraph{Results.} \Cref{table:oracle_vs_without_oracle_ER2,table:oracle_vs_without_oracle_ER4} present the mean and standard deviation of RMSE and unfairness over 20 synthetic datasets.

Overall, using oracle CPDAGs tends to improve fairness performance, especially in terms of unfairness, because it removes errors from causal graph estimation and allows each method to construct adjustment sets from the correct Markov equivalence class.
This trend is observed for both variable-level methods, $\epsilon$-\textbf{IFair} and $\ell$-\textbf{IFair}, and our cluster-level method, \textbf{C-IFair}.
The improvement is more apparent in settings where causal discovery is more challenging, such as nonlinear datasets or denser ER(4) graphs.
These results support the importance of accurate causal graph information for achieving interventional fairness under graph uncertainty.

\subsection{Sensitivity Analyses} 

This section presents two sensitivity analyses: (1) the sensitivity of our \textbf{C-IFair} method to the choice of cluster sizes, and (2) the sensitivity to the IPW weight estimation and Mellowmax temperature parameter used in our unfairness evaluation procedure.

\subsubsection{Cluster Sizes} \label{subsec-add-sensitivity-cluster-size}

To investigate the sensitivity of our \textbf{C-IFair} method to the choice of cluster sizes, we evaluate its performance on synthetic datasets with different cluster sizes.

\paragraph{Data and Causal Graphs.} We use the nonlinear synthetic datasets generated by unbounded nonlinear functions  as described in \Cref{subsec-synthdata}. 

We randomly generate the variable-level causal graphs from the ER(2) model, and then construct cluster-level causal graphs by grouping variables into clusters based on the pre-defined cluster partition. 

We fix the number of variables $d_v$ to 30, and vary the number of clusters $d$ to 5, 10, and 15 by adjusting the cluster sizes to 6, 3, and 2, respectively.
Since this cluster-size adjustment also changes the number of variables in the sensitive cluster $\biA$ and the outcome cluster $\biY$, 
the performance of \textbf{all} methods can be affected by the choice of cluster size.

\paragraph{Results.} \Cref{table:synth_rmse_unfairness_cluster} presents the mean and standard deviation of RMSE and unfairness over 20 nonlinear synthetic datasets with different cluster sizes.

Overall, \textbf{C-IFair} achieves low unfairness across all cluster sizes, 
indicating that its fairness performance is reasonably robust to the choice of cluster granularity.
In particular, \textbf{C-IFair} obtains the lowest unfairness for cluster size $6$.
For cluster sizes $3$ and $2$, it achieves substantially lower RMSE than $\epsilon$-\textbf{IFair} and $\ell$-\textbf{IFair}, while maintaining comparable unfairness.

\begin{table*}[t]
  \centering
  \caption{Sensitivity analysis of cluster size on 20 nonlinear synthetic datasets generated by unbounded nonlinear functions (quadratic, cubic, and Sigmoid Linear Unit (SiLU)). Cluster size denotes the number of variables in each cluster and is set to 6, 3, and 2 for $d=5$, $d=10$, and $d=15$, respectively.}
  \label{table:synth_rmse_unfairness_cluster}
  \scshape

  \scalebox{0.92}{%
  \begin{tabular}{lcccccc}
    \toprule
    & \multicolumn{2}{c}{$d$=5 (cluster size 6)} & \multicolumn{2}{c}{$d$=10 (cluster size 3)} & \multicolumn{2}{c}{$d$=15 (cluster size 2)} \\
    \textsc{Method}
    & \textsc{RMSE}$\downarrow$ & \textsc{Unfairness}$\downarrow$
    & \textsc{RMSE}$\downarrow$ & \textsc{Unfairness}$\downarrow$
    & \textsc{RMSE}$\downarrow$ & \textsc{Unfairness}$\downarrow$ \\
    \midrule
    \textsc{Oracle} 
      & $0.910 \pm 0.149$ & $0.000 \pm 0.000$
      & $0.761 \pm 0.085$ & $0.000 \pm 0.000$
      & $0.648 \pm 0.118$ & $0.000 \pm 0.000$ \\
    \midrule
    \textsc{Full} 
      & $0.434 \pm 0.089$ & $0.262 \pm 0.226$
      & $0.526 \pm 0.087$ & $0.080 \pm 0.070$
      & $0.431 \pm 0.133$ & $0.016 \pm 0.013$ \\
    \midrule
    \textsc{Unaware} 
      & $0.718 \pm 0.181$ & $0.048 \pm 0.036$
      & $0.716 \pm 0.082$ & $0.031 \pm 0.009$
      & $0.605 \pm 0.108$ & $0.012 \pm 0.010$ \\
    \textsc{No-DesCs} 
      & $0.724 \pm 0.184$ & $0.036 \pm 0.024$
      & $0.611 \pm 0.070$ & $0.028 \pm 0.014$
      & $0.504 \pm 0.116$ & $0.013 \pm 0.012$ \\
    $\epsilon$-\textsc{IFair} 
      & $0.719 \pm 0.113$ & $0.015 \pm 0.012$
      & $0.671 \pm 0.045$ & $0.033 \pm 0.031$
      & $0.628 \pm 0.105$ & $0.013 \pm 0.011$ \\
    $\ell$-\textsc{IFair} 
      & $0.797 \pm 0.564$ & $0.039 \pm 0.026$
      & $0.813 \pm 0.057$ & $0.049 \pm 0.038$
      & $0.784 \pm 0.088$ & $0.015 \pm 0.019$ \\
    \textsc{C-IFair} 
      & $0.680 \pm 0.164$ & $0.012 \pm 0.013$
      & $0.556 \pm 0.060$ & $0.022 \pm 0.012$
      & $0.475 \pm 0.116$ & $0.012 \pm 0.009$ \\
    \bottomrule
  \end{tabular}%
  }
\end{table*}

\subsubsection{IPW Weight Estimation and Mellowmax Temperature Parameter} \label{subsec-add-sensitivity}

Our \textbf{C-IFair} method measures the unfairness of predictions by employing the weighted kernel MMD estimators and taking their maximum over adjustment sets. This unfairness evaluation procedure involves two mechanisms: (\myra) the use of IPW weights for estimating interventional distributions, and (\myrb) the use of the Mellowmax function \citep{asadi2017alternative} as a differentiable approximation of the maximum function.

To analyze the sensitivity to these two mechanisms, we evaluate the performance of our \textbf{C-IFair} method by varying the IPW clipping threshold and the Mellowmax temperature parameter $\omega$.

\textbf{Data.} We use exactly the same linear synthetic datasets as described in \Cref{subsec-synthdata}.

\textbf{Evaluation Protocol.} To evaluate the sensitivity to the IPW weights, we perform the clipping of (self-normalized) IPW weights. We clip the IPW weights using their $q$-th quantile in each mini-batch, where $q \in \{0.95, 0.975, 0.99\}$ and compare the performance with the case without clipping (i.e., $q=1.0$).

To test the sensitivity to the Mellowmax function, 
we vary the temperature parameter $\omega$ of the Mellowmax function, which controls the smoothness of the approximation
and approaches the maximum function as $\omega$ increases.
We evaluate the performance with $\omega \in \{2.0, 10.0, 100.0\}$ and compare it with the variant of our method that does not use Mellowmax but takes the maximum over adjustment sets. 

As described in \Cref{subsec-setup}, 
in our main experiments, we do not perform IPW clipping and use the Mellowmax function with $\omega=10.0$ as a default setting.

\textbf{Results.} \Cref{fig:IPW-clipping,fig:Mellowmax-temperature} present the average performance of our \textbf{C-IFair} method on 20 linear synthetic datasets when the IPW clipping threshold and the Mellowmax temperature parameter $\omega$ are varied.

We observe that the performance of our \textbf{C-IFair} method is not very sensitive to the IPW clipping or to the Mellowmax temperature value. However, our \textbf{C-IFair} method consistently achieves lower RMSE than the variant without Mellowmax (illustrated as the black dashed line in \Cref{fig:Mellowmax-temperature}), indicating that the use of Mellowmax function as differentiable approximation can help stabilize the gradient-based training and strike better balance between accuracy and fairness.

	\begin{figure}[t]
		\includegraphics[height=6cm]{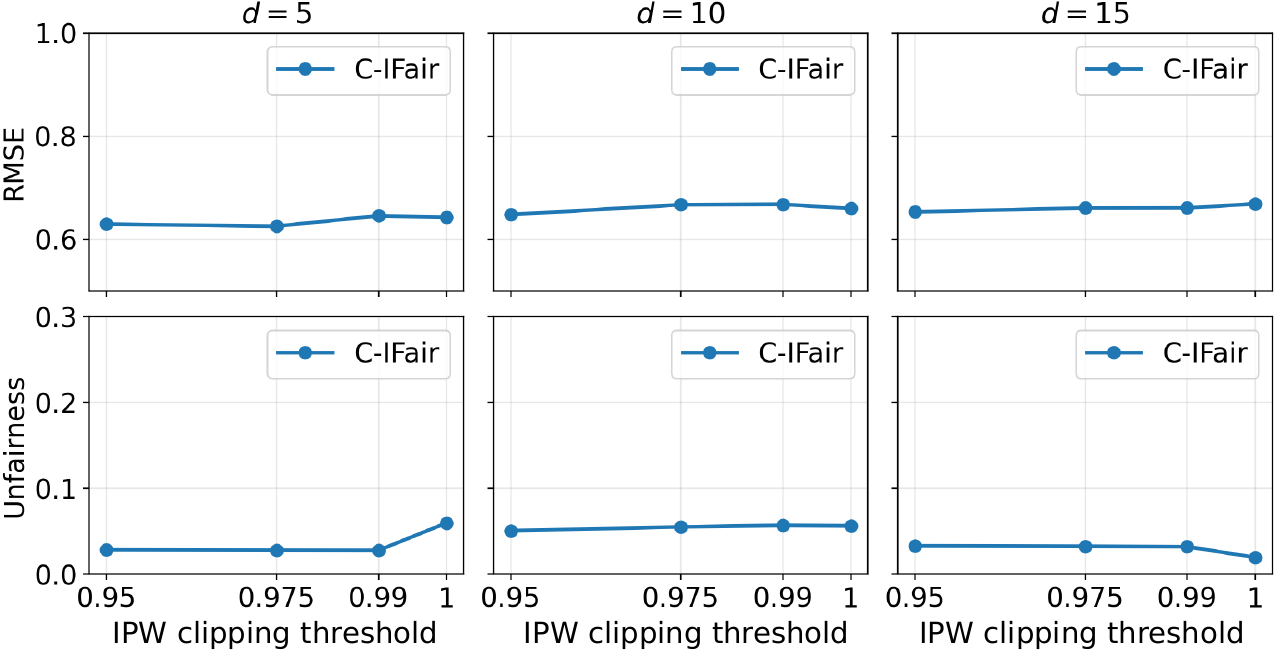}
		\centering 
			\caption{Effect of IPW clipping threshold on performance  on linear synthetic datasets with $d=5, 10, 15$ clusters}
		\label{fig:IPW-clipping}
	\end{figure}
  	\begin{figure}[t]
		\includegraphics[height=6cm]{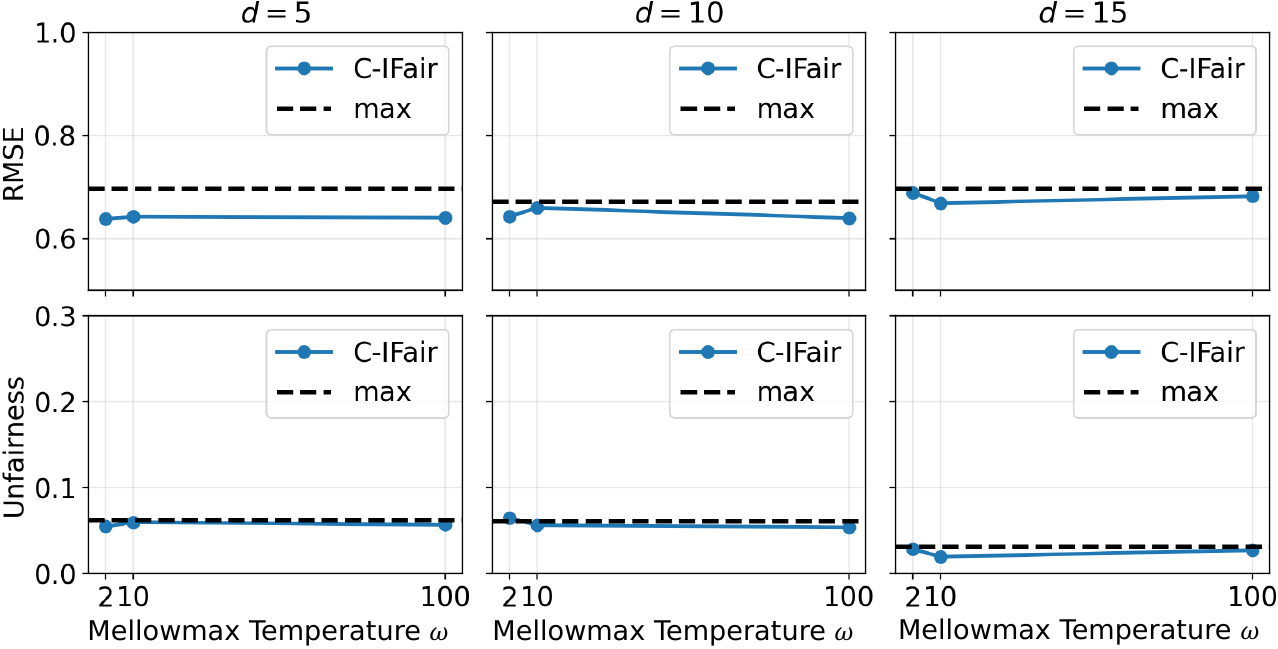}
		\centering 
			\caption{Effect of Mellowmax temperature on linear synthetic datasets with $d=5, 10, 15$ clusters. Black dashed line denotes variant of our \textbf{C-IFair} method that takes maximum over adjustment sets without using Mellowmax.}
		\label{fig:Mellowmax-temperature}
	\end{figure}

\section{Further Discussion on Cluster Formation and Predictive Modeling}
\label{app:cluster-formation}

\subsection{Relation between Admissible Partitions and Tier-Based Background Knowledge}

Our framework assumes that the user-specified cluster partition is admissible, i.e., the induced cluster-level graph does not contain directed cycles.
This assumption is related to, but different from, tier-based background knowledge commonly used in causal discovery \citep{andrews2020completeness,bang2023wiser}.

Tier-based background knowledge can be formalized as a partition of variables into ordered tiers $\biT_1,\dots,\biT_K$.
It requires that, for any directed edge $X_i\to X_j$ in the underlying DAG, if $X_i\in \biT_k$ and $X_j\in \biT_l$, then $k\leq l$.
Equivalently, directed edges from later tiers to earlier tiers are forbidden.
Such knowledge can help users construct a cluster partition that is likely to be admissible; for example, variables measured earlier in time, demographic attributes, intermediate socioeconomic variables, and outcome-related variables may naturally form different tiers.

However, admissibility does \textbf{not} require the user to provide such ordered tiers in advance.
It only requires that the induced cluster-level graph admits \textbf{some} acyclic ordering.
Thus, tier-based background knowledge is one practical way to guide cluster formation, but it is \textbf{not} identical to the admissible partition assumption.
Recent work also studies this distinction, showing that Cluster-DAG-based background knowledge can represent groupwise causal constraints beyond tier-based background knowledge \citep{vargas2026clusterdags}.

In our setup, forming clusters based on user-specified sensitive and admissible features must not introduce directed cycles in the cluster-level graph.
This is a limitation of the current fairness framework based on admissible partitions.
As discussed in \Cref{sec:related_work}, recent theoretical work has begun to relax partition admissibility in cluster DAGs \citep{yvernes2026relaxing}.
Extending our framework to such more general cluster causal representations is an important direction for future work.

\subsection{Toward More Effective Predictive Modeling}

Whereas our method does not exploit local causal structure inside the outcome cluster $\biY$, 
if we have access to the domain knowledge about such local structure, we can potentially improve predictive performance by leveraging specialized predictive models (e.g., graphical models using the local graph knowledge).
Developing such predictive models while preserving interventional fairness guarantees is an interesting direction for future work.

\end{document}